%% file: main_preprint.tex
\documentclass{article}

\usepackage[preprint, final]{neurips_2026}

\usepackage[utf8]{inputenc} 
\usepackage[T1]{fontenc}    
\usepackage{hyperref}       
\usepackage{url}            
\usepackage{booktabs}       
\usepackage{amsfonts}       
\usepackage{nicefrac}       
\usepackage{microtype}      
\usepackage{xcolor}         
\usepackage{tikz}
\usepackage{graphicx}
\usepackage{subcaption}
\usepackage{algorithm}
\usepackage{algpseudocode}
\usepackage{amsmath}
\usepackage{amssymb}
\usepackage{mathtools}
\usepackage{amsthm}
\usetikzlibrary{positioning}
\usepackage{wrapfig}
\usepackage{enumitem}
\usepackage{makecell}
\usepackage{multirow}
\usepackage{booktabs}
\usepackage{threeparttable}
\usepackage[capitalize,noabbrev]{cleveref}

\theoremstyle{plain}
\newtheorem{theorem}{Theorem}[section]
\newtheorem{proposition}[theorem]{Proposition}

\newtheorem{corollary}[theorem]{Corollary}
\theoremstyle{definition}
\newtheorem{definition}[theorem]{Definition}

\theoremstyle{remark}
\newtheorem{remark}[theorem]{Remark}

\title{Stable GFlowNets with TV Monitoring and Probabilistic Guarantees}

\author{%
  Zengxiang Lei$^{1}$, Ananth Shreekumar$^{2}$, Jonathan Rosenthal$^{2}$, Ruoyu Song$^{2}$ \\ \textbf{Alvaro A. Cardenas$^{3}$, Daniel J. Fremont$^{3}$, Dongyan Xu$^{2}$} \\ \textbf{Satish Ukkusuri$^{1,2*}$, Z. Berkay Celik$^{2*}$}\\[0.5em]
  $^{1}$Lyles School of Civil and Construction Engineering,
  Purdue University, IN, USA \\
  $^{2}$Department of Computer Science,
  Purdue University, IN, USA \\
  $^{3}$Computer Science and Engineering,
  University of California, Santa Cruz, USA \\[0.5em]
  $^{*}$Corresponding authors:
  \texttt{sukkusur@purdue.edu},
  \texttt{zcelik@purdue.edu}
}

\begin{document}

\maketitle

\begin{abstract}
Generative Flow Networks (GFlowNets) sample diverse structured objects in proportion to reward and have been applied to molecular discovery and biological-sequence design, where finding multiple high-quality candidates is more useful than returning a single optimum. Despite their theoretical promise, practical training is often unstable, exhibiting severe loss spikes and mode collapse. To address this, we first assess the sensitivity of GFlowNet objectives, demonstrating that a small Total Variation (TV) distance between the learned and target distributions does not preclude an unbounded training loss. Motivated by this mismatch, we establish converse guarantees by deriving loss-to-TV bounds that certify global fidelity from bounded trajectory balance losses. Lastly, we propose \emph{Stable GFlowNets}, which leverages our theoretical results to stabilize training via adaptive reference flow and improves the trade-off among mode coverage, robustness, and certifiability. Code is available at: \url{https://github.com/tjleizeng/StableGFlowNets}.
\end{abstract}

\input{text/introduction}

\input{text/background}

\input{text/problem}

\input{text/related_work_comprehensive}

\input{text/evaluation}

\input{text/limitation}

\section*{Acknowledgment}

This work is based upon the work supported by the National Center for
Transportation Cybersecurity and Resiliency (TraCR) (a U.S. Department of
Transportation National University Transportation Center) headquartered at
Clemson University, Clemson, South Carolina, USA. Any opinions, findings,
conclusions, and recommendations expressed in this material are those of the author(s) and do not necessarily reflect the views of TraCR, and the U.S.
Government assumes no liability for the contents or use thereof. The authors declare no competing interests.

{
\small
\bibliographystyle{unsrtnat}
\bibliography{bib/references_cleaned}
}


\include{text/appendix_new}



\end{document}

%% file: text/introduction.tex
\section{Introduction}
Generative Flow Networks (GFlowNets) provide a principled framework for learning generative policies to sample states according to a target, unnormalized reward function~\citep{bengio2021flow}. By modeling generation as a sequential decision process and enforcing flow consistency constraints, GFlowNets enable efficient sampling from complex combinatorial spaces and have shown promise in molecular design, biological sequence discovery, combinatorial optimization, and adversarial generation~\citep{bengio2021flow,jain2022biological,shen2023tacogfn,zhang2023lettheflows,lahlou2023theory}.

Despite their successes, training GFlowNets remains challenging in practice. Prior work reports numerical instability~\citep{madantowards, malkin2022trajectory} and difficulty in reliably capturing rare but high-reward modes~\citep{pan2022gafn, kim2023local}. These behaviors contrast with the comparatively stable and scalable optimization of other likelihood-based generative models (e.g., diffusion models~\citep{rombach2022high}), and raise a fundamental question: 
\emph{What guarantees on sampling fidelity can we certify with GFlowNet training?}

A central difficulty lies in the ambiguous relationship between non-zero GFlowNet losses and global sampling error. Unlike in reinforcement learning (RL) and diffusion models, where small training loss often yields theoretical guarantees that bound policy performance~\citep{singh1994upper,sutton1998reinforcement,song2021maximum}, common GFlowNets training objectives, including flow matching (FM), detailed balance (DB), and trajectory balance (TB), guarantee correctness only at their global optima (i.e., when loss is zero everywhere)~\citep{bengio2023gflownet, malkin2022trajectory}. 
In realistic regimes with finite data, function approximation, and non-stationary policies, optimization is necessarily approximate. Furthermore, we demonstrate in Section~\ref{sec:increment_learning} that high loss signals can persist even when the learned policy is near-optimal. Consequently, it is unclear whether observed loss values meaningfully reflect distributional fidelity, or whether controlling the loss is sufficient to bound the mismatch between the learned state distribution and the target reward distribution.

In this work, we close this gap by establishing a rigorous connection between GFlowNet training losses and the global Total Variation (TV) error, i.e., the total variation distance between the learned distribution and the reward-proportional target. 
We show, in an analytically tractable setting, that small TV error does not imply bounded training loss, so extreme optimization signals may arise even when the learned distribution is globally accurate. We then derive loss-to-TV bounds that quantify global fidelity from bounded TB losses, and we further provide finite-sample TV certificates via trajectory sampling. Since loss spikes can be inevitable, we analyze \emph{reference flow}, a customized flow injected into the flow conservation constraints to prevent extreme ratios and thus cap GFlowNet loss,
as a stabilization mechanism that reduces loss magnitudes without altering the global objective, and quantify the resulting stability-fidelity trade-off. Building on these insights, we propose \emph{Stable GFlowNets}, which uses adaptive reference flows to stabilize incremental learning and provide global or subgraph-level TV certificates depending on the available sampling oracle. Our contributions are summarized as follows:

\begin{enumerate}
    \item We characterize the sensitivity of GFlowNet objectives and prove that a low Total Variation (TV) error does not imply bounded training loss.
    \item We derive the first series of GFlowNet training loss-to-TV bounds that connect training losses to the fidelity of the learned distribution relative to the target distribution.
    \item We introduce \emph{Stable GFlowNets}, which use adaptive reference flows to stabilize training and improve mode coverage while supporting TV monitoring and probabilistic certification.
\end{enumerate}

%% file: text/background.tex
\section{Preliminaries}
\label{sec:background}

\begin{wrapfigure}{R}{0.43\textwidth}
    \centering
    \includegraphics[width=\linewidth]{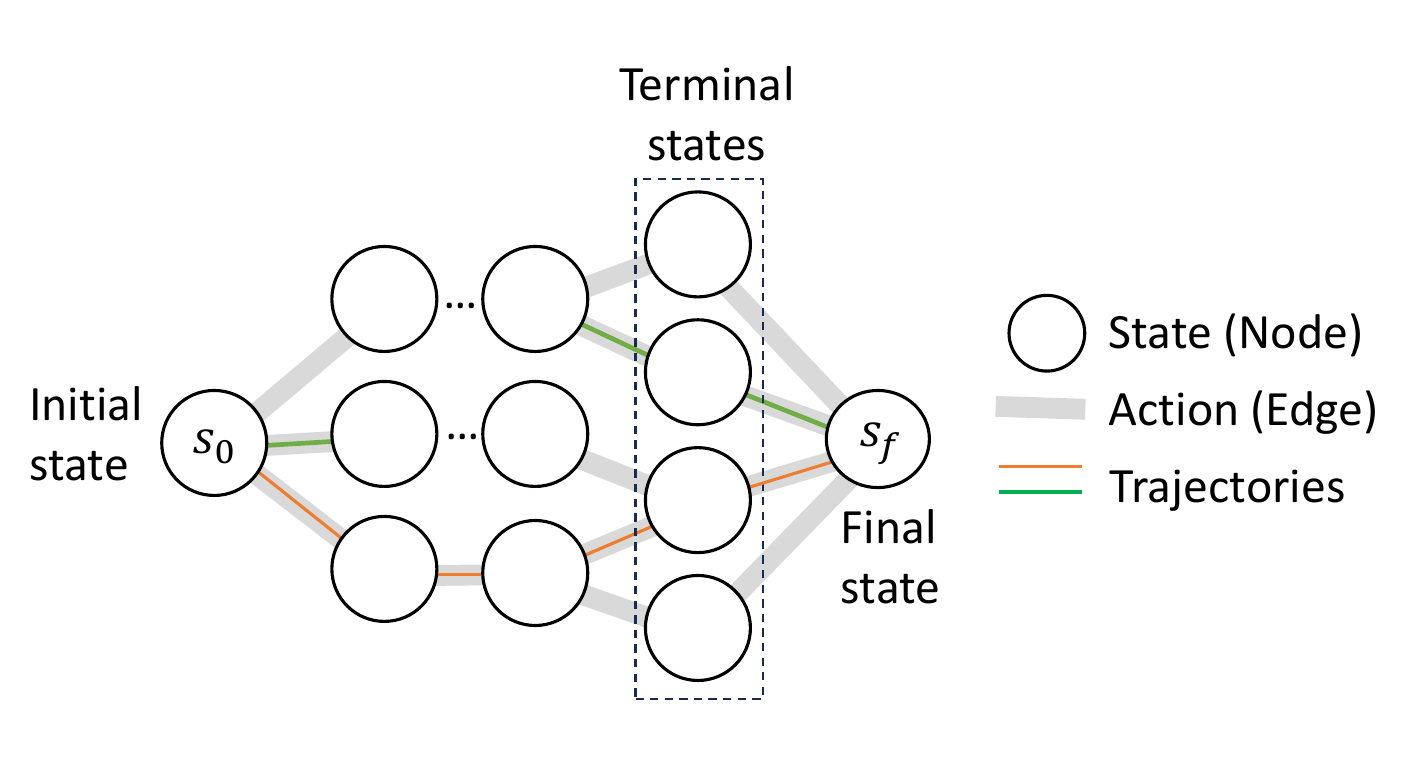}
    \caption{An illustrative GFlowNet DAG.}
    \label{fig:small_dag}
\end{wrapfigure}

A GFlowNet represents generation as a finite-state DAG (Figure~\ref{fig:small_dag}): each node (or state) is a partially constructed object, each edge (or action) is a valid construction action, and each terminal state is a complete object with reward. Training encourages the total flow reaching each terminal state to be proportional to its reward, so that high-reward objects are sampled more often without collapsing to a single solution. In molecule or biological-sequence generation, for example, nodes correspond to partial molecules or sequences, and terminal nodes correspond to completed candidates. A formalized description is provided below.

\paragraph{Notation.} We consider a finite directed acyclic graph (DAG) $\mathcal{G} = (\mathcal{S}, \mathcal{A})$ where the set of nodes forms the state space $\mathcal{S}$ and the set of directed edges forms the action space $\mathcal{A}$. The unique state with no incoming edges is the initial state $s_{0} \in \mathcal{S}$, and the unique state with no outgoing edges is the final state $s_{f} \in \mathcal{S}$. States that transition only to $s_{f}$ are referred to as terminal (or end) states; their set is denoted by $\mathcal{X}\subseteq \mathcal{S}$. A (complete) trajectory is defined as a sequence of states $\tau = (s_{0} \rightarrow \cdots \rightarrow s_{n} \rightarrow s_{n+1} = s_f)$ where the terminal state $s_n \in \mathcal{X}$ and each action (or edge) $s_{i} \rightarrow s_{i + 1} \in \mathcal{A}$. In contrast, any prefix of such a sequence is a subtrajectory. The set of all complete trajectories is denoted by $\mathcal{T}$. Each trajectory $\tau$ is assigned a positive flow value $F: \tau \rightarrow \mathbb{R}_{+}$, which induces state-level flow and edge-level flow as $F(s) = \Sigma_{\tau \in \mathcal{T}, s \in \tau}{F(\tau)}$ and $F(s, s') = \Sigma_{\tau \in \mathcal{T}, s \rightarrow s' \in \tau}{F(\tau)}$, respectively. 
Each edge is further evaluated by a forward policy $P_F:\mathcal{S} \times \mathcal{S} \rightarrow (0,1)$ and a backward policy $P_B:\mathcal{S} \times \mathcal{S} \rightarrow (0,1)$; when evaluated on a specific edge $s\rightarrow s'$, they define the transition probability $P_F(s, s') = P(s' \mid s)$ and $P_B(s, s') = P(s \mid s')$. A positive scalar $Z$ (the partition function, if treated as learnable) relates flow to policy via $F(\tau) = ZP_F(\tau)= Z \Pi_{t=0}^n P_F(s_{t}, s_{t+1})$, and we have the marginal distribution of the flow over terminal states $\mathcal{X}$ as $P_T(x) = F(x)/Z$. 
The reward satisfies $R(s)>0$ for all $s\in\mathcal{X}$ and $R(s)=0$ for $s\notin\mathcal{X}$.
A table of notation is available in Appendix~\ref{app:notation}.

\paragraph{GFlowNets.} Depending on the specific parameterization, GFlowNets learn either an edge flow $F$ or a state flow $F$ (or a partition function $Z$) along with a forward policy $P_F$ (and optionally with a backward policy $P_B$). The training objectives ensure that the generated flow on the DAG $\mathcal{G}$ results in a sampling probability proportional to the reward, $F(x) \propto R(x)$ for all $x\in \mathcal{X}$. 
The backward policy is optional because the GFlowNet framework permits $P_B$ to be chosen arbitrarily~\citep{bengio2023gflownet}. However, fixing $P_B$ constrains the forward policy $P_F$ to a unique solution, whereas learning $P_B$ jointly with $P_F$ can empirically accelerate convergence~\citep{malkin2022trajectory}.
The choice between learning $F$ versus learning $Z$ depends on the training objective: objectives such as trajectory balance estimate $Z$ jointly with $P_F$, while others learn flows $F$ and take $Z=F(s_0)$ implicitly. Given the terminal state space $\mathcal{X}$, the true partition function $Z^*$ can be computed explicitly as $Z^* = \sum_{x\in\mathcal{X}} R(x)$ and the ground-truth target distribution is $\pi_{target}(x) =  \frac{R(x)}{Z^*}$.

\paragraph{Training Objectives.} GFlowNets can be trained using several objectives that enforce the global flow-matching condition $F(x) \propto R(x)$ through different forms of local consistency. These objectives differ in where the consistency constraint is localized: on individual edges, on flow conservation at states, or on trajectories. 

\begin{itemize}[leftmargin=*, nosep]
\item \textbf{Flow Matching (FM).}~\citet{bengio2021flow} mandate flow conservation at each intermediate state: 
\begin{equation}
    \mathcal{L}_{FM} (s') =  \left(\log\frac{\sum_{s \rightarrow s' \in \mathcal{A}} F(s, s')}{R(s') + \sum_{s' \rightarrow s'' \in \mathcal{A},s''\neq s_f} F(s',s'')}\right)^2
\end{equation}
In~\citet{bengio2021flow}, a hyperparameter $\delta$ was introduced to mitigate numerical issues. For simplicity in our theoretical analysis, we omit $\delta$ here and discuss its role in Section~\ref{sec:existing_solution}.
\item \textbf{Detailed Balance (DB).}~\citet{bengio2023gflownet} enforce local consistency on individual edges by equating forward and backward flows:
\begin{equation}
    \mathcal{L}_{DB} (s, s') = \left( \log \frac{ F(s)P_F(s, s')}{F(s')P_B(s,s')}\right)^2
\end{equation}
where $s\rightarrow s' \in \mathcal{A}$, $s'\neq s_f$, and $F(s') := R(x)$ when $s' = x \in \mathcal{X}$. 

\item \textbf{Trajectory Balance (TB).}~\citet{malkin2022trajectory} impose consistency on complete trajectories $\tau = (s_0 \rightarrow s_1  \rightarrow \dots  \rightarrow s_n=x \rightarrow s_f)$:
\begin{equation}
    \mathcal{L}_{TB} (\tau) = \left( \log \frac{Z \Pi_{t=0}^{n-1} P_F(s_{t},s_{t+1})}{R(x) \Pi_{t=0}^{n-1} P_B(s_{t}, s_{t+1})} \right)^2
\end{equation}

\item \textbf{Subtrajectory balance (subTB).}~\citet{madan2023learning} extend TB to subtrajectories $\tau^{partial} = (s_{t_1} \rightarrow \dots \rightarrow s_{t_2})$: 
\begin{align}
    &\mathcal{L}_{subTB} (\tau_{t_1: t_2}) \nonumber \\&= \left( \log \frac{F(s_{t_1}) \Pi_{t=t_1}^{t_2-1} P_F(s_{t},s_{t+1})}{F(s_{t_2}) \Pi_{t=t_1}^{t_2-1} P_B(s_{t}, s_{t+1})} \right)^2
\end{align}
where $F(s_{t_2}) := R(x)$ when $s_{t_2} = x \in \mathcal{X}$.
\end{itemize}

\paragraph{Total Variation (TV) Error.} Total Variation (TV) error quantifies the distributional discrepancy and is central to recent theoretical analyses of GFlowNets~\citep{silva2025gflownets}. 
The TV error corresponds to half the total $L_1$ error:
\begin{equation}
    \mathrm{TV}(P_T,\pi_{target}) = \frac{1}{2}\sum_{x \in \mathcal{X}} \big| P_T(x) - \pi_{target}(x) \big|
\end{equation}

%% file: text/problem.tex
\section{On the Sensitivity and Fidelity of GFlowNets}

Our analysis is motivated by incremental mode coverage, where a GFlowNet must incorporate new high-reward modes without catastrophic forgetting. Section~\ref{sec:increment_learning} shows that standard GFlowNet objectives are ill-conditioned in this setting: small target-distribution perturbations can induce unbounded local losses through worst-case local contrast ratios. This motivates the reverse question: what distributional fidelity can be guaranteed when training losses are bounded? Section~\ref{sec:connection_between_loss_and_TV} answers this by linking training loss to TV error (Theorem~\ref{thm:tb_error_tv_bound}) and deriving a trajectory-sampling-based probabilistic certificate (Theorem~\ref{thm:bigtheorem}). Section~\ref{sec:reference_flow} then introduces reference flow to smooth local contrast and stabilize training. Although reference flow caps loss magnitudes, it imposes an inherent fidelity trade-off: a multiplicative degradation governed by the augmentation magnitude (Theorem~\ref{thm:resolution_reference_flow}). Finally, Theorem~\ref{thm:probablistic_TV_bound_w_reference_flow} provides a probabilistic certificate for reference-flow-stabilized GFlowNets, directly motivating the training procedure in Section~\ref{sec:solution}.

\subsection{Incremental Mode Coverage Causes Loss Explosion}
\label{sec:increment_learning}

To lay the groundwork, we examine the simplified setting in Figure~\ref{fig:regular_tree}. The following remarks illustrate a sharp mismatch between global flow error and local training behavior. 
Remark~\ref{rmk:flow_error} shows that adding a reward to a single leaf in a large tree ($g^h$ leaves) results in a negligible Total Variation error of $TV(P_T, \pi_{target}) \approx 1/g^h$. However, Remark~\ref{rmk:training_loss} reveals that the local loss scales as $(\log \epsilon)^2$. As $\epsilon \to 0$, the training loss diverges to $+\infty$, even though the global distribution is nearly correct.

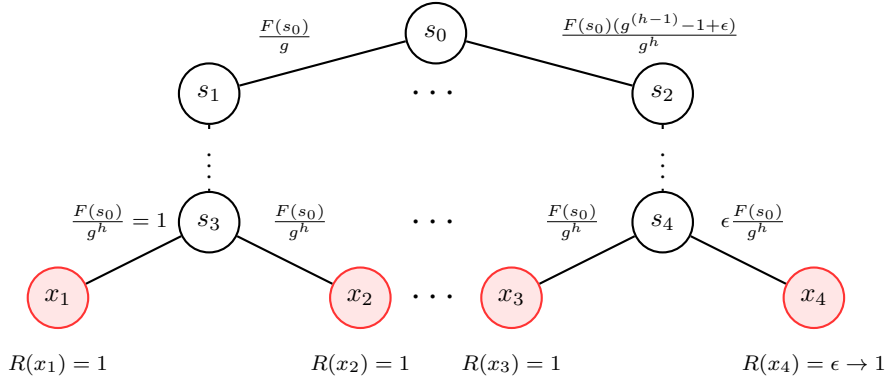
\begin{figure}[!ht]
\centering
\begin{tikzpicture}[
    state/.style={circle, draw=black, thick, minimum size=0.8cm, fill=white},
    terminal/.style={circle, draw=red!80, thick, minimum size=0.8cm, fill=red!10},
    edge/.style={->, >=, thick},
    flow/.style={font=\footnotesize\sffamily, text=black}
]

    \node[state] (s0) at (0, 3.5) {$s_0$};

    \node[state] (s1) at (-3, 2.7) {$s_1$};
    \node[font=\Large] at (0, 2.7) {$\cdots$}; 
    \node[state] (s2) at (3, 2.7) {$s_2$};

    \node[state] (s3) at (-3, 1) {$s_3$};
    \node[font=\Large] at (0, 1) {$\cdots$};
    \node[state] (s4) at (3, 1) {$s_4$};

    \node[terminal] (x1) at (-5, 0) {$x_1$};
    \node[terminal] (x2) at (-1, 0) {$x_2$};
    \node[font=\Large] at (0, 0) {$\cdots$};
    \node[terminal] (x3) at (1, 0) {$x_3$};
    \node[terminal] (x4) at (5, 0) {$x_4$};
    

    \draw[edge] (s0) -- (s1) node[midway, above left, flow] {$\frac{F(s_0)}{g}$};
    \draw[edge] (s0) -- (s2) node[midway, above right, flow]{$\frac{F(s_0) (g^{(h-1)} - 1 + \epsilon)}{g^h}$};

    \draw[edge, dashed] (s1) -- (s3) node[midway, fill=white] {$\vdots$};
    \draw[edge, dashed] (s2) -- (s4) node[midway, fill=white] {$\vdots$};

    \draw[edge] (s3) -- (x1);
    \draw[edge] (s3) -- (x2);
    
    \draw[edge] (s4) -- (x3);
    \draw[edge] (s4) -- (x4);


    \node[text=black, align=center, font=\footnotesize] at (-4.2, 1.) {$\frac{F(s_0)}{g^h} = 1$};
    \node[text=black, align=center, font=\footnotesize] at (-1.8, 1.) {$\frac{F(s_0)}{g^h}$};
    
    \node[text=black, align=center, font=\footnotesize] at (1.8, 1) {$\frac{F(s_0)}{g^h}$};
    \node[text=black, align=center, font=\footnotesize] at (4.2, 1) {$\epsilon\frac{F(s_0)}{g^h}$};

    \node[below=0.2cm of x1, font=\footnotesize] {$R(x_1)=1$};
    \node[below=0.2cm of x2, font=\footnotesize] {$R(x_2)=1$};
    \node[below=0.2cm of x3, font=\footnotesize] {$R(x_3)=1$};
    \node[below=0.2cm of x4, font=\footnotesize] {$R(x_4)= \epsilon \rightarrow 1$};

\end{tikzpicture}
\caption{The ``one more mode'' learning setup on a  $g$-ary tree of depth $h$. The underlying structure is a regular tree. We initialize the experiment with an already fitted model where this specific node leads to a negligible reward $R(x_4) = \epsilon$ ($\ll 1$). To introduce the new mode, we update the reward function by promoting $x_4$ to a high-reward state with  $R(x_4)=1$. The objective is to learn this new mode while preserving the previously acquired ones.}
\label{fig:regular_tree}
\end{figure}



\begin{remark}{Flow error for ``one more mode'' learning over a regular tree.}
\label{rmk:flow_error}
\begin{equation}
\mathrm{TV}(P_T, \pi_{target}) = \frac{(g^h-1)}{2}(\frac{1}{g^h-1+\epsilon} - \frac{1}{g^h}) + \frac{1}{2}(\frac{1}{g^h} - \frac{\epsilon}{g^h - 1+\epsilon}) \approx \frac{1}{g^h}
\end{equation}

\end{remark}

\begin{remark}{Training loss for ``one more mode'' learning over a regular tree.} 
\label{rmk:training_loss}

Most losses are zero; non-zero losses are shown below and are all $(\log \epsilon )^2$.

\begin{equation}
\mathcal{L}_{FM}(x_4) = \mathcal{L}_{DB}(s_4, x_4) = \mathcal{L}_{TB}(
\tau\ni x_4)  = \mathcal{L}_{subTB}(\tau^{partial} \ni x_4) = (\log \epsilon) ^ 2 
\end{equation}
\end{remark}

This implies that the GFlowNet training signal does not necessarily reflect the scale of the distributional mismatch. Instead, it can be dominated by worst-case local contrast ratios (i.e., states where the learned flow is orders of magnitude smaller than what the target implies).

For a general \textbf{incremental mode coverage} setting, we assume that the network has converged to an initial reward landscape $R_{prev}(x)$ such that $F(x) = R_{prev}(x)$, and we are introducing a new reward $R_{new}(x) = R_{prev}(x) + R'(x)$ where $R'(x) \geq 0$. 

\begin{proposition}{(TV bound for incremental mode coverage over arbitrary state graph).}
\label{thm:tv_bound_incremental}
Let the reward added at each state in a subset $\mathcal{X}_{sub}$ be $R' (x)$. We define the local true partition function $Z^*_{\mathcal{Y}\subseteq \mathcal{X}} = \sum_{x \in \mathcal{Y}} R(x)$ and the local contrast ratio $\Lambda_\mathcal{Y} = \frac{Z^*_\mathcal{Y}}{Z^*_\mathcal{Y} + \sum_{x\in\mathcal{Y}} R'(x)}$. Then, we have 
\begin{align}
   \frac{Z^* - Z^*_{\mathcal{X}_{sub}}}{Z^*} (1-\Lambda_{\mathcal{X}}) \leq & \mathrm{TV}(P_T, \pi_{target})  \leq (1-\Lambda_{\mathcal{X}})
\end{align}
\end{proposition}

Here, note that $Z^{*} = Z_{\mathcal{X}}^{*}$ is the true partition function on the original reward function.

\begin{proposition}{(Loss scale via local contrast).}
\label{thm:contrast_ratio}
While the TV error depends on the aggregate $\Lambda_{\mathcal{X}}$, the supremum of the training loss is governed by the \textbf{worst-case} local contrast ratio, given by:
\begin{equation}
\sup \mathcal{L}_{GFN} = \left(\log \min_{\{x\}\subseteq \mathcal{X}_{sub}} \Lambda_{\{x\}}\right) ^ 2
\end{equation}
where $GFN \in \{FM, DB, TB, subTB\}$.
\end{proposition}

Proofs of Propositions~\ref{thm:tv_bound_incremental} and~\ref{thm:contrast_ratio} are given in Appendix~\ref{app:proof_of_propositions}. Together, they show that incremental mode coverage in GFlowNets is governed by two quantities: the aggregate contrast ratio, e.g., $1-\Lambda_{\mathcal{X}}$, which controls global TV error and may remain small when the new mode has little reward mass; and the worst-case local contrast ratio, $\min_{\{x\}\subseteq \mathcal{X}_{sub}} \Lambda_{\{x\}}$, which controls the scale of losses and gradients through terms such as $(\log \min_{\{x\}} \Lambda_{\{x\}})^2$. Thus, when a new mode requires a large relative reward increase at rarely visited states, the training dynamics degenerate into a regime where tiny changes in the target distribution cause huge optimization signals.

\subsection{The General Link Between GFlowNets Training Loss and TV Error}
\label{sec:connection_between_loss_and_TV}

As TV error does not effectively bound GFlowNet training losses, we pursue the reverse direction to bound the TV error of the resulting policy if the training loss can be contained, yielding the following theorems for the general setting.

\begin{theorem}{(Training loss to TV distance bound).}
\label{thm:tb_error_tv_bound}
The relationship between the training loss bound and the resulting TV distance depends on the scope of the objective (trajectory-level vs. transition-level):

\textbf{Trajectory-level Objective.}
If the trajectory loss is bounded, i.e., $\mathcal{L}_{TB}(\tau) \leq c^2, \forall \tau \in \mathcal{T}$, the TV error is bounded by:
\begin{equation}
\mathrm{TV}(P_T, \pi_{target}) \leq 1 - e^{-2c}
\end{equation}
This bound is independent of the trajectory length, as TB optimizes the full path consistency directly.

\textbf{Transition-level Objective.}
If the local transition loss is bounded, i.e., $\mathcal{L}_{DB}(s, s') \leq c^2$ or $\mathcal{L}_{FM}(s') \leq c^2$, the global consistency relies on the accumulation of local estimates. For trajectories of maximum number of transitions $L$, the error bound degrades linearly with depth in the log-domain:
\begin{equation}
\mathrm{TV}(P_T, \pi_{target}) \leq 1 - e^{-2Lc}
\end{equation}
\end{theorem}

Since DB and FM bounds depend on the maximum trajectory length, we focus our analysis on the TB loss. Extensions to the DB and FM cases follow directly by introducing the trajectory length into the bounds. The proof can be found in Appendix~\ref{app:proof_of_propositions}.

As verifying the loss $\mathcal{L}_{TB} \leq c^2$ for each trajectory is intractable, we develop a probabilistic certificate using random sampling. Notably, this certificate is independent of the state-space size.

\begin{theorem}{(Probabilistic TV bound via trajectory sampling).}
\label{thm:bigtheorem}
Given the ground-truth target distribution $\pi_{target}$, we define a target distribution over trajectories $\hat{\pi}(\tau) = \pi_{target}(x_\tau) P_B(\tau | x_\tau)$. Sample $m$ trajectories $\tau_1, \dots, \tau_m$ from $\hat{\pi}$ independently by sampling $x \sim \pi_{target}$ and $\tau \sim P_B(\cdot|x)$. Sample another $n$ trajectories independently using $P_F$. Let $c = \max_{i\leq m+n} \sqrt{\mathcal{L}_{TB}(\tau_i)}$, with confidence $1-2\alpha$, the global TV error is bounded by:
\begin{equation}
\label{equ:probablistic_TV}
\mathrm{TV}(P_T, \pi_{target}) \leq e^{2c} + 1 - \alpha^\frac{1}{m} -  \alpha^\frac{1}{n} \leq e^{2c} - 1 + \frac{\log(1/\alpha)}{m} +\frac{\log(1/\alpha)}{n}\end{equation}
\end{theorem}

The proof is provided in Appendix~\ref{app:proof_of_propositions}. When sampling end states from the full ground-truth target distribution is replaced by sampling from a subset $\mathcal{X}_{sub}\subseteq \mathcal{X}$, the theorem yields a \emph{subgraph certificate} over $\mathcal{X}_{sub}$.

\begin{corollary}{(Subgraph Certification via trajectory sampling).}
\label{thm:bigcorollary}
Let $\mathcal{X}_{sub} \subseteq \mathcal{X}$ be a subset of end-states. Define the restricted target distribution over $\mathcal{X}_{sub}$ by $\pi_{target}^{sub}(x) = R(x)/\sum_{x\in\mathcal{X}_{sub}}R(x)$, the corresponding restricted target trajectory distribution as $\hat{\pi}_{sub}(\tau) = \pi^{sub}_{target}(x_\tau) P_B(\tau | x_\tau)$. Sample $m$ trajectories $\tau_1, \dots, \tau_m$ from $\hat{\pi}_{sub}$. Sample another $n$ trajectories that end within $\mathcal{X}_{sub}$ independently using $P_F$. Suppose we observe $\mathcal{L}_{TB}(\tau_i) \leq c^2$ for all trajectories in both sets. Let $P_T^{\mathrm{sub}}$ denote the terminal flow $P_T$ renormalized to $\mathcal{X}_{\mathrm{sub}}$.
Then, with confidence $1-2\alpha$,
\begin{equation}
\label{equ:probablistic_TV_sub}
\mathrm{TV}(P_T^{sub}, \pi^{sub}_{target}) \leq e^{2c} + 1 - \alpha^\frac{1}{m} -  \alpha^\frac{1}{n} \leq e^{2c} - 1 + \frac{\log(1/\alpha)}{m} +\frac{\log(1/\alpha)}{n} \end{equation}
\end{corollary}


\textbf{Throughout the paper, when global backward sampling is replaced by subgraph-based sampling, the certificate applies only to the corresponding subgraph.}

\subsection{Reference Flow: Stability with Fidelity Trade-off}
\label{sec:reference_flow}

The bounds in Section~\ref{sec:connection_between_loss_and_TV} are only useful if the training loss can be kept small, a requirement that could be easily violated as seen in Section~\ref{sec:increment_learning}. To resolve this, we investigate \textbf{reference flow}, which artificially increases the background flow to reduce the training losses. 

\begin{definition}{(Trajectory reference flow).}
For a trajectory $\tau$, let $R(\tau) = R(x)P_B(\tau|x)$ be the target flow. We introduce a trajectory-specific, non-negative reference flow $\delta(\tau) \geq 0$ to augment the existing one. The augmented flow $F_{aug}(\tau)$ and augmented target $R_{aug}(\tau)$ are defined as:
$F_{aug}(\tau) = Z P_F(\tau) + \delta(\tau), \quad R_{aug}(\tau) = R(\tau) + \delta(\tau)$
\end{definition}

\begin{remark}{(Stabilization via reference flow).} 
Let $\mathcal{L}_{TB}(\tau) = \left( \log \frac{Z P_F(\tau)}{R(\tau)} \right)^2$ be the standard TB loss. The reference flow proportionally reduces the scale of a positive TB loss:
\begin{equation}
\label{equ:augmented_loss}
\mathcal{L}_{aug}(\tau) =\left(\log\frac{F_{aug}(\tau)}{R_{aug}(\tau)}\right)^2= \frac{1}{\gamma^2} \mathcal{L}_{TB}(\tau)
\qquad \gamma := \frac{
    \log(ZP_F(\tau)/R(\tau))
}{
    \log[(ZP_F(\tau)+\delta(\tau))/(R(\tau)+\delta(\tau))]
}
\end{equation}
\end{remark}

It is easy to verify that $\gamma^2 > 1$ when $\delta(\tau)>0$. To make $\mathcal{L}_{aug}(\tau)\leq c^2$, we have the minimum reference flow:
\begin{equation} 
\label{equ:reference_flow}
\delta_c(\tau) = \begin{cases} \frac{Z P_F(\tau) - e^cR(\tau)}{e^c - 1} & \text{if } \log \frac{Z P_F(\tau)}{R(\tau)} > c  \\
\frac{R(\tau) - e^{c}Z P_F(\tau)}{e^{c}-1} & \text{if } \log \frac{Z P_F(\tau)}{R(\tau)} < -c \\
0 & \text{otherwise} 
\end{cases} 
\end{equation}

Treating reference flows as target modifications enables Theorem~\ref{thm:tb_error_tv_bound} to expose the following trade-off:

\begin{theorem}{(Fidelity trade-off under reference flow).}
\label{thm:resolution_reference_flow}
The fidelity of the recovered policy depends on the ratio between the training loss and the augmentation magnitude. Let the total reference flow be $\Delta = \sum_{\tau \in \mathcal{T}} \delta(\tau)$. If the reference training loss is bounded by $\mathcal{L}_{aug} (\tau) \leq c^2$, the terminal distribution $P_T$ induced by the learned forward policy satisfies:
\begin{equation}
\mathrm{TV}(P_{T}, \pi_{target}) \leq \frac{(1 - e^{-2c})(1+\Delta/Z^*)}{1 + (1-e^{-c})\Delta/Z^*} \leq  (1 - e^{-2c})(1+\frac{\Delta}{Z^*}) \end{equation}
\end{theorem}

Note $\frac{\Delta}{Z^*}$ can be expressed as the expectation of $\frac{\delta(\tau)}{R(\tau)}$ over the target distribution $\hat{\pi}(\tau) = \frac{R(\tau)}{Z^*}$, \textbf{we can approximate the TV bound via Monte Carlo estimation} ($\mathcal{M}_{\text{TV}}$):
\begin{equation}
\label{equ:mc_tv_bound}
    \mathcal{M}_{\text{TV}} := \min_{c\in \left(0, \max_i \sqrt{\mathcal{L}_{TB}(\tau_i)}\right]}(1 - e^{-2c}) \left( 1 + \frac{1}{m} \sum_{i=1}^{m} \frac{\delta_c(\tau_i)}{R(\tau_i)} \right), \quad \tau_i \sim \hat{\pi}
\end{equation}

This observation also yields the following probabilistic bound.

\begin{theorem}{(Probabilistic TV bound with optimizable reference-flow threshold).}
\label{thm:probablistic_TV_bound_w_reference_flow}
Sample $m(\geq2)$ trajectories $\tau_1, \dots, \tau_m$ independently from the target $\hat{\pi}$, sample another $n(\geq2)$ trajectories  $\tau_{m+1}, \dots, \tau_{m+n}$ independently using $P_F$. For each $c>0$, we compute the minimum reference flow $\delta_c(\tau_i)$ according to Equation~ (\ref{equ:reference_flow}). Define $M_c:=\max_{i\in\{1,\dots,m+n\}}\frac{\delta_c(\tau_i)}{R(\tau_i)}$, and $\mathcal{C} = \{c>0\mid M_c < \frac{1}{e^c-1}\}$. With confidence $1-2\alpha$, the following bound holds simultaneously for every $c\in\mathcal{C}$:
\begin{align}
\label{equ:probablistic_TV_with_reference_flow}
\mathrm{TV}(P_{T}, \pi_{target}) &\leq  \left(\frac{e^c + (e^c-1)M_c}{e^{-c} - (1-e^{-c}) M_c} - 1\right) + \beta_m(\alpha) + \beta_n(\alpha)\\ & \leq  \left(\frac{e^c + (e^c-1)M_c}{e^{-c} - (1-e^{-c}) M_c} - 1\right) + \frac{2\log(2/\alpha)}{m} + \frac{2\log(2/\alpha)}{n} \nonumber
\end{align}
where $\beta_k(\alpha)$ is solved from \begin{equation}
    (1 - \beta_k(\alpha))^{k-1}[1+(k-1)\beta_k(\alpha)] = \alpha
\end{equation}
\end{theorem}

Since $c$ can be selected within a range in the above theorem, we define 
\begin{equation}
\label{equ:bet_TV}
\mathcal{B}_{\text{TV}} := \min_{c\in\mathcal{C}}\left(\frac{e^c + (e^c-1)M_c}{e^{-c} - (1-e^{-c}) M_c} - 1\right) + \frac{2\log(2/\alpha)}{m} + \frac{2\log(2/\alpha)}{n} 
\end{equation}

\subsection{Applications to GFlowNets Training: Stable GFlowNets}
\label{sec:solution}

\paragraph{Practical Implications of Our Analysis.} Our theoretical analysis leads to several practical implications for GFlowNet training. First, our results indicate that training stability requires explicit attention and is likely to become increasingly important as the state space grows and rewards become sparser. Second, Theorem~\ref{thm:bigtheorem} shows that reliably optimizing trajectories beyond those sampled by the current forward policy, such as backward-sampled trajectories, is important for obtaining probabilistic performance guarantees. This perspective also helps explain the effectiveness of GFlowNet training variants that use guided exploration or replay to improve coverage of high-reward trajectories (see Appendix~\ref{app:extended_related_work} for an extended discussion). Third, $\mathcal{M}_{\text{TV}}$ can serve as an efficient training-time monitoring signal when exact TV is computationally prohibitive, whereas $\mathcal{B}_{\text{TV}}$ provides a conservative, high-confidence certificate of model quality.

\paragraph{The Stable GFlowNets Algorithm.} These implications directly motivate our Stable GFlowNets algorithm (Algorithm~\ref{alg:stabilized_gflownet}), which adaptively injects a reference flow $\delta(\tau)$ based on instantaneous mismatch and monitors the resulting probabilistic TV certificate. When exact backward sampling over $\mathcal{X}$ is intractable, certification is restricted to a subgraph $\mathcal{X}_{\text{sub}}\subset\mathcal{X}$, implemented as a top-$K$ high-reward buffer; in practice, a modest $K$ often captures most of the reward mass. Full implementation details are deferred to Appendix~\ref{app:stable_gfn}, including the exponential moving-average update for the loss threshold $c$ and the bounded 1D optimization used to compute $\mathcal{B}_{\mathrm{TV}}$ and $\mathcal{M}_{\mathrm{TV}}$.

%% file: text/related_work_comprehensive.tex
\section{Related Work}\label{sec:related_work}
\label{sec:literature}

\paragraph{Theoretical Analysis of GFlowNets.} Theoretical understanding of GFlowNets has grown substantially since their formulation as flow-matching models on directed acyclic graphs~\citep{bengio2023gflownet}. Trajectory Balance (TB)~\citep{malkin2022trajectory} investigated long-horizon credit assignment, while subsequent work connected GFlowNets to continuous generative modeling, diffusion models~\citep{lahlou2023theory}, variational inference~\citep{malkin2022gflownets}, and entropy-regularized reinforcement learning~\citep{tiapkin2024entropy}.

Recent theoretical work has advanced performance assessment for GFlowNets. \citet{krichel2024generalization} relate the average TB loss to its sampled estimate, while \citet{silva2025gflownets} connect flow perturbations to total variation (TV) error and introduce Flow Consistency in Sub-graphs (FCS) as a scalable evaluation metric. However, FCS is not directly tied to the GFlowNet training loss, distinguishing their setting from ours; we empirically compare FCS with $\mathcal{M}_{\text{TV}}$ in Section~\ref{sec:rq3}. Furthermore, \citet{silva2025generalization} bounded TV using trajectory-level violations $(\log \frac{P_F(\tau)}{\pi(x)P_B(\tau\mid x)})^2$. This reliance on the true target distribution $\pi(x)$ contrasts with our focus on unconverged training losses based on learned quantities, where $\frac{Z}{R(x)} \neq \frac{1}{\pi(x)}$.

\paragraph{Theoretical Analysis of RL and Diffusion.} In related domains like RL \citep{singh1994upper,schulman2015trust} and diffusion \citep{song2021maximum}, bounded training objectives directly yield theoretical performance guarantees. GFlowNets combine distribution matching (as in diffusion) with active exploration for high-reward states (as in RL), yet differ from both: they train without a fixed dataset and optimize distributional fidelity rather than reward maximization. This makes it nontrivial to determine when training-time samples faithfully represent global objectives.

\paragraph{GFlowNets Training.}
\label{sec:existing_solution}

We focus on methods relevant to our stable GFlowNets algorithm (see Appendix~\ref{app:extended_related_work} for a broader discussion). To mitigate persistent numerical instability, prior works add a constant $\delta$ to transition flows \citep{bengio2021flow,bengio2023gflownet}, or clip losses \citep{lahlou2023torchgfn} and gradients \citep{shen2023towards}. However, clipping lacks bounded loss certificates. Furthermore, Theorem~\ref{thm:resolution_reference_flow} shows that a fixed $\delta$ introduces a resolution loss that can render TV guarantees vacuous without fully stabilizing training. Efficient exploration is also critical, traditionally promoted via forward-policy annealing \citep{lahlou2023torchgfn}, transition augmentations \citep{pan2022gafn}, local backtracking \citep{kim2023local}, or adaptive teachers \citep{kim2024adaptive}. We reframe exploration as a tool for \textit{certification}: backward trajectories sampled from the reward distribution and $P_B$ bound global TV error, emphasizing the necessity of discovering and visiting high-reward trajectories.

%% file: text/evaluation.tex
\section{Experiments}
\label{sec:evaluation}

\begin{table}[t]
    \setlength{\tabcolsep}{3pt}
    \small
    \centering
    \caption{Overview of the research questions, evaluation purposes,  evidence, and main finding.}
    \label{tab:exp_summary}
    \resizebox{\linewidth}{!}{%
    \begin{tabular}{p{3.5cm} p{3.7cm} p{4.7cm} p{4.2 cm}}
        \toprule
        \textbf{RQ} &
        \textbf{Purpose} &
        \textbf{Evidence} &
        \textbf{Main finding} \\
        \midrule

        RQ1: How severe is the loss imbalance in standard GFlowNet training? &
        Characterize practical loss instability during training. &
        \textbf{Tasks:} All environments. 
        \textbf{Methods:} DB, FM, and TB. 
        \textbf{Metrics:} Training loss and Max-to-Rest Loss Ratio. &
        A few trajectories can dominate the training objective even as the aggregate loss decreases. \\

        \addlinespace 

        RQ2: Does the proposed Stable GFlowNet improve training stability, convergence, and mode coverage? &
        Evaluate Stable GFlowNet and isolate the contributions of its components under matched sampling budgets. &
        \textbf{Tasks:} Hypergrid, L14-RNA1, and sEH. 
        \textbf{Methods:} DB, TB, SubTB, WDB, Teacher, Stable, and StableTeacher. 
        \textbf{Metrics:} $L_1$ error, diversity, and variability across seeds. &
        Stable GFlowNet improves training robustness and sampling quality, particularly on more challenging tasks. \\

        \addlinespace

        RQ3: How informative are the TV bounds derived in Theorems~3.10 and~3.11? &
        Evaluate global and subgraph-level monitoring of distributional error. &
        \textbf{Tasks:} Regular Tree, Hypergrid, and L14-RNA1. 
        \textbf{Metrics:} True TV, FCS, $\mathcal{M}_{\mathrm{TV}}$, and $\mathcal{B}_{\mathrm{TV}}$. &
        $\mathcal{M}_{\mathrm{TV}}$ tracks distributional error closely, whereas $\mathcal{B}_{\mathrm{TV}}$ provides a conservative certificate. \\

        \bottomrule
    \end{tabular}
    }%
    \vspace{-0.5em}
\end{table}

\begin{figure}
    \centering
    \includegraphics[width=1.0\linewidth]{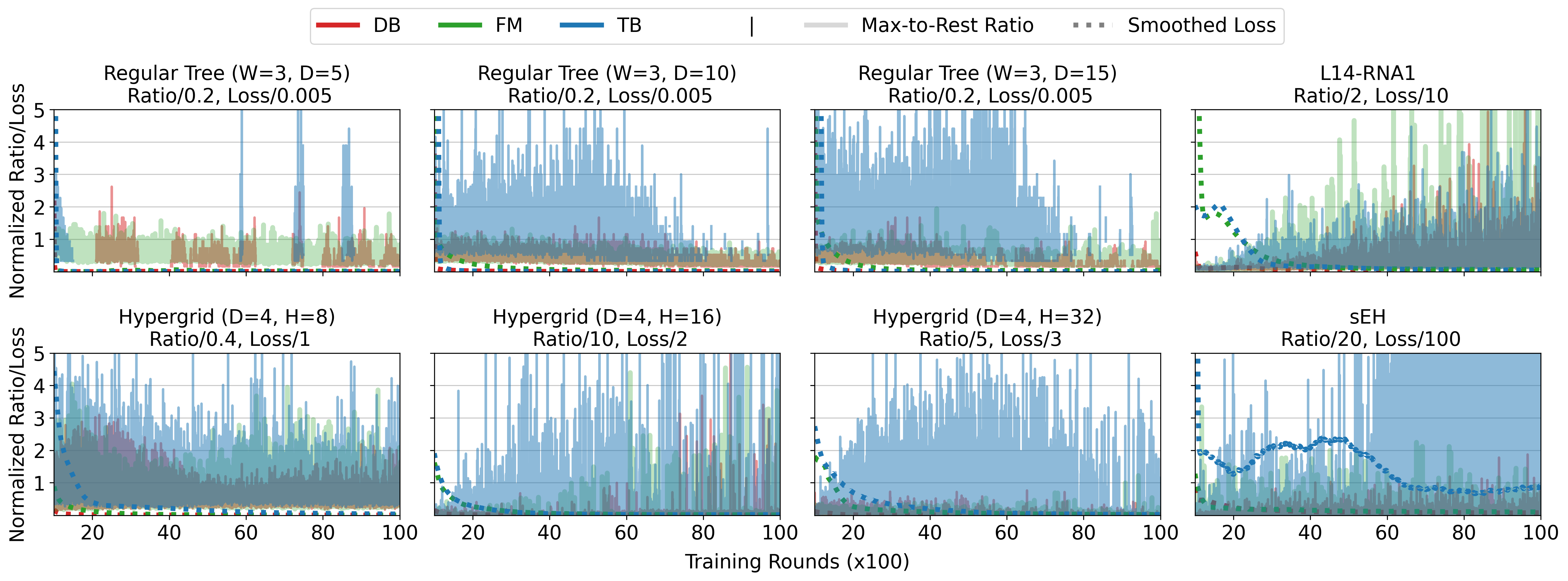}
    \caption{Loss Concentration and Training Stability. Solid lines show the Max-to-Rest Loss Ratio, defined as $\frac{\max_i \mathcal{L}_{GFN}(\tau_i)}{\sum_{j\neq i}\mathcal{L}_{GFN}(\tau_j)}$, and dotted lines show the training loss smoothed over 1,000 steps. Large spikes indicate that a single trajectory can contribute more loss than all remaining trajectories combined, revealing severe trajectory-level imbalance despite the steadily decreasing smoothed loss.}
    \label{fig:exp1_result}
    \vspace{-0.5em}
\end{figure}

\begin{table}[t]
\centering
\begin{threeparttable}
\caption{Performance comparison. All methods are trained for $10^4$ rounds, and Hypergrid evaluation uses $10^5$ samples from the final policy $P_F$. RNA mode counts and sEH scaffold counts are accumulated
over training. All results are mean $\pm$ std over 5 seeds.}
\label{tab:performance}
\begin{tabular}{@{}c@{}}
\resizebox{\linewidth}{!}{%
\begin{tabular}{l*{4}{r@{\,$\pm$\,}l}}
\toprule
& \multicolumn{4}{c}{\textbf{Hypergrid} (Empirical Total $L_1$ $\downarrow$)}
& \multicolumn{2}{c}{\textbf{L14-RNA1}}
& \multicolumn{2}{c}{\textbf{sEH}} \\
\cmidrule(lr){2-5} \cmidrule(lr){6-7} \cmidrule(lr){8-9}
\textbf{Method}
& \multicolumn{2}{c}{$D=4, H=16$}
& \multicolumn{2}{c}{$D=4, H=32$}
& \multicolumn{2}{c}{$\#$ modes $\uparrow$}
& \multicolumn{2}{c}{$\#$ scaffolds $\uparrow$} \\
\midrule
TB
& 1.885 & {\small 0.000}
& 1.875 & {\small 0.000}
& 1669.2 & {\small 300.0}
& 137.2 & {\small 77.9} \\
DB
& 0.315 & {\small 0.004}
& 0.944 & {\small 0.017}
& 1285.2 & {\small 16.4}
& 32.2 & {\small 12.5} \\
SubTB
& 0.292 & {\small 0.001}
& 0.749 & {\small 0.005}
& 1904.0 & {\small 54.0}
& 68.4 & {\small 41.9} \\
WDB
& 0.490 & {\small 0.037}
& 1.510 & {\small 0.024}
& 1364.6 & {\small 64.5}
& 3.2 & {\small 2.8} \\
Teacher
& 0.407 & {\small 0.043}
& 1.603 & {\small 0.083}
& 1991.0 & {\small 189.3}
& 5.6 & {\small 5.9} \\
\midrule
\textbf{Stable}
& \textbf{0.290} & \textbf{{\small 0.002}}
& \textbf{0.713} & \textbf{{\small 0.002}}
& 1734.0 & {\small 22.4}
& \textbf{14142.6} & \textbf{{\small 2388.2}} \\
\textbf{StableTeacher}
& 0.315 & {\small 0.004}
& 0.815 & {\small 0.010}
& \textbf{2622.8} & \textbf{{\small 131.4}}
& 3722.2 & {\small 4187.8} \\
\textbf{TB + backward sampling$^\dagger$}
& 0.316 & {\small 0.003}
& \textbf{0.713} & \textbf{{\small 0.003}}
& 1740.8 & {\small 37.2}
& 8680.2 & {\small 2254.2} \\
\bottomrule
\end{tabular}%
}
\end{tabular}

\begin{tablenotes}[flushleft]
\footnotesize
\item[$\dagger$]
Internal ablation of Stable GFlowNets that retains backward
sampling but disables reference-flow stabilization.
\end{tablenotes}
\end{threeparttable}
\vspace{-0.5em}
\end{table}

\paragraph{Environments.} We evaluate our methods across four environments:
\begin{itemize}[leftmargin=*, nosep]
    \item \textbf{Regular Tree:} A 3-ary tree of depth $D$. Each leaf node corresponds to a terminal state and receives a unit reward. This environment is simple to train and allows exact computation of TV error, making it well-suited for validating our theory.
    
    \item \textbf{Hypergrid:} A grid-based environment introduced by~\citet{bengio2021flow}. It is parameterized by dimension $D$, side length $H$, and three reward coefficients $R_0$, $R_1$, and $R_2$.
    We use $R_{0} = 10^{-2\log_2(H/8)-1}$, $R_{1} = 0.5$, and $R_{2} = 2.0$.
    
    \item \textbf{L14-RNA1:} The generated objects are RNA sequences of length $14$. The reward function is a binding affinity to a human transcription factor, obtained via a pre-trained proxy model from~\citet{sinai2020adalead}. Following~\citet{kim2024adaptive}, we use a reward exponent of $40$ and define modes as the top $0.01\%$ quantile of $R(x)$. Diversity filtering with a Levenshtein distance threshold of 1 is enforced, resulting in 8,967 modes out of 268,435,456 possible end states.
    
    \item \textbf{sEH:} Following~\citet{bengio2021flow}, this environment involves generating small molecule graphs targeting the soluble epoxide hydrolase (sEH) protein. Molecules are constructed step-by-step using a vocabulary of molecular building blocks, with up to 105 actions available per state. This combinatorial process results in a massive state space of approximately $10^{16}$ reachable terminal states. Performance is evaluated by counting the number of distinct Bemis–Murcko scaffolds among molecules whose raw reward is at least $7.5$.
\end{itemize}

\textbf{Baselines.} For RQ1, we examine loss behavior under standard GFlowNet objectives: TB~\citep{malkin2022trajectory}, DB~\citep{bengio2023gflownet}, and FM~\citep{bengio2021flow}. For the RQ2 performance
comparison, we evaluate TB, DB, SubTB~\citep{madan2023learning}, Adaptive Teacher
(Teacher)~\citep{kim2024adaptive}, and Weighted DB
(WDB)~\citep{silva2025gflownets}.
We exclude FM from RQ2 because its flow-conservation loss
requires aggregating incoming flows over all parents of
each visited state,  involving additional
parent-state evaluations outside the sampled trajectories. To separate the contributions of backward sampling and
reference-flow stabilization, we include an internal ablation,
``TB + backward sampling'', which disables reference flow in
Stable GFlowNets while keeping the remaining training procedure
unchanged. All methods considered in RQ2 use a batch size of 32 and are trained
for $10^4$ rounds. We match the number of training trajectories,
while allowing method-specific mixtures of forward, backward,
and teacher sampling; the sampling compositions are detailed
in Appendix~\ref{app:hyperparams}. For L14-RNA1, we employ reward-prioritized replay~\citep{shen2023towards} and $\epsilon$-greedy exploration~\citep{jain2022biological}. For backward sampling, we use the ground-truth terminal state distribution on the Regular Tree and Hypergrid; on L14-RNA1 and sEH, we utilize a buffer of the top-10{,}000 highest-reward states discovered during training. The correspondence among environments, baselines, and RQs is summarized in Table~\ref{tab:exp_summary}.

\subsection{Diagnosing Loss Instabilities (RQ1)}

Figure~\ref{fig:exp1_result} shows that batch losses can become heavily concentrated on a single trajectory, and that this concentration generally increases with state size. One exception is the medium-sized Hypergrid, which exhibits larger Max-to-Rest Loss Ratio spikes than the larger Hypergrid and we investigate this behavior in Appendix~\ref{app:hypergrid_explain}.

\subsection{Stability and Mode Coverage Improvements  (RQ2)}
Table~\ref{tab:performance} shows that Stable GFlowNets matches the best baselines on easier Hypergrid tasks and shares the lowest reported mean with TB + backward sampling on the harder $H=32$ setting, and discovers substantially more scaffolds than all baselines on the large-scale sEH task. On L14-RNA1, StableTeacher discovers the most modes
during training. Stable discovers a similar number of
modes as TB with backward sampling, with lower variability
across five seeds, although its average discovery count
is below those of Teacher and SubTB. On sEH, Stable
achieves the highest average scaffold count among
the tested methods. 

The comparison among standard TB, our backward sampling ablation,
and Stable separates the gains from backward sampling from those
provided by reference flow stabilization. Replacing half of the forward-sampled trajectories with backward-sampled trajectories is already a powerful training mechanism, consistent with our theory that capping losses on both forward- and backward-sampled trajectories leads to stronger  performance guarantees. 
On Hypergrid and L14-RNA1, backward sampling accounts for most of the improvement, while Stable yields more consistent performance across seeds. On the substantially larger sEH task with much sparser rewards, however, stabilization provides its largest empirical gain: Stable discovers $14,142.6$ scaffolds, compared with $ 8,680.2$ for TB with backward sampling, corresponding to a $63\%$ improvement.

\subsection{Interpreting Global and Subgraph TV Certificates (RQ3)}
\label{sec:rq3}

\begin{figure}
    \centering
    \includegraphics[width=1.0\linewidth]{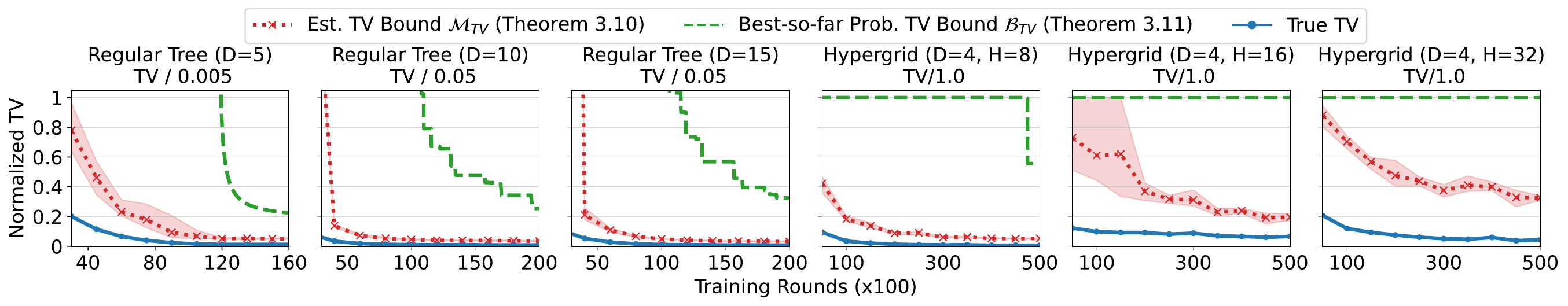}
    \caption{Derived TV bounds vs. true TV error. $\mathcal{M}_{\text{TV}}$ follows the magnitude and trend of the observed TV error, providing an informative monitoring signal throughout training, whereas $\mathcal{B}_{\text{TV}}$ provides a formal but more conservative probabilistic certificate. Failure would appear as an uncorrelated or oppositely varying trend, or as $\mathcal{B}_{\text{TV}}$ underestimating the true TV error more frequently than permitted by its confidence level. Shaded regions show the min-max range over 5 independent trials of one trained policy, with $\mathcal{M}_{\text{TV}}$ estimated using 100 backward-sampled trajectories per trial. $\mathcal{B}_{\text{TV}}$ uses the accumulation procedure during training.}
    \label{fig:exp3_result}
    \vspace{-0.5em}
\end{figure}

\begin{table}
\setlength{\tabcolsep}{3pt}
\centering
\caption{Probabilistic certification and monitoring performance across environments. The best $\mathcal{B}_{\text{TV}}$ is the smallest probabilistic TV certificate across 10 evenly spaced training checkpoints ($10\%$ to $100\%$). Monitoring performance is assessed by the correlation of $\mathcal{M}_{\text{TV}}$ and FCS~\citep{silva2025gflownets} with global true TV across checkpoints. For calibration, a linear mapping from $\mathcal{M}_{\text{TV}}$ to true TV is fitted on the first five checkpoints and evaluated by predictive RMSE on the remaining five; the RMSE of FCS is computed on the same held-out checkpoints.}
\label{tab:correlation_calibration}
\resizebox{0.8\textwidth}{!}{
\begin{tabular}{l c c c c c}
\toprule
\textbf{Environment}
& \makecell{Best\\$\mathcal{B}_{\text{TV}}$}
& \makecell{Corr.($\mathcal{M}_{\text{TV}}$,\\true TV)}
& \makecell{Corr.(FCS,\\true TV)}
& \makecell{RMSE of calibrated\\$\mathcal{M}_{\text{TV}}$}
& \makecell{RMSE of\\FCS} \\
\midrule
RegularTree $D=5$
& $0.0747$
& $0.999$
& $0.996$
& $2.85 \times 10^{-5}$
& $1.24 \times 10^{-5}$ \\

RegularTree $D=10$
& $0.0802$
& $0.959$
& $0.999$
& $0.0004$
& $1.48 \times 10^{-5}$ \\

RegularTree $D=15$
& $0.0812$
& $0.915$
& $1.000$
& $0.0009$
& $1.89 \times 10^{-5}$ \\

Hypergrid $D=4,\ H=8$
& $0.2542$
& $0.997$
& $0.996$
& $0.0033$
& $0.0012$ \\

Hypergrid $D=4,\ H=16$
& $1.0000$
& $0.969$
& $0.931$
& $0.0088$
& $0.0397$ \\

Hypergrid $D=4,\ H=32$
& $1.0000$
& $0.977$
& $0.980$
& $0.0068$
& $0.0044$ \\

L14-RNA1
& $1.0000$
& $0.930$
& $0.628$
& $0.0231$
& $0.0302$ \\
\bottomrule
\end{tabular}
}
\vspace{-0.5em}
\end{table}

Table~\ref{tab:correlation_calibration} shows that $\mathcal{B}_{\text{TV}}$ provides nontrivial probabilistic certificates on Regular Tree and the smaller Hypergrid setting. In contrast, $\mathcal{M}_{\text{TV}}$  correlates strongly with
the reported TV values across checkpoints.
After linear calibration, it achieves lower held-out
RMSE than FCS on Hypergrid with $H=16$ and L14-RNA1,
whereas FCS has lower RMSE in the remaining settings. However, FCS requires the probabilities that the model assigns to individual terminal states, whereas $\mathcal{M}_{\text{TV}}$ uses only backward-sampled trajectories and does not require these probabilities. We also evaluate FCS using Monte Carlo estimates of the terminal state probabilities instead of their exact values; see Table~\ref{tab:metrics2} in Appendix~\ref{app:theorem_additional_results}.

%% file: text/limitation.tex
\section{Conclusion and Discussion}\label{sec:limitation}
We establish loss-to-TV guarantees and finite-sample TV certificates. Building on these theoretical foundations, we introduce Stable GFlowNets, which adapt reference flows to stabilize training while supporting TV-based monitoring and certification.
Despite this progress, our guarantees currently rely on backward sampling, which could be difficult in continuous settings, and the probabilistic certificate can be overly conservative due to its worst-case dependence on the reference flow. Although the certificate remains theoretically valid and may become nonvacuous as training converges, developing tighter and more practically informative certificates remains an important direction for future work. Other directions include extending Theorem~\ref{thm:bigtheorem} to continuous state settings and exploring latent representations as alternative stabilization mechanisms with comparable guarantees.

%% file: text/appendix_new.tex
\appendix
\section{Notation}
\label{app:notation}
\begin{table}[htbp]
\centering
\renewcommand{\arraystretch}{1.3} 
\begin{tabular}{@{} l p{12cm} @{}}
\toprule
\textbf{Symbol} & \textbf{Definition} \\
\midrule
$\mathcal{S}, s_{0}, s_{f}$ & Set of all states, initial state, and final state. \\
$\mathcal{A}$ & Set of all actions (edges); $s \rightarrow s' \in \mathcal{A}$ indicates a transition with positive support. \\
$\mathcal{X}$ & Set of terminal states, i.e., states which can only transition to $s_f$. \\
$\mathcal{T}$ & Set of complete trajectories (i.e., starting at $s_{0}$ and ending at $s_{f}$) with positive support. \\
$F(s), F(\tau)$ & Total flow through a state $s$ or a trajectory $\tau$. \\
$F(s \rightarrow s')$ & Total flow through an edge $s \rightarrow s'$. \\
$P_F$ & Forward policy, assigning probabilities to edges or trajectories; for $\tau=(s_0,s_1,\dots,s_n,s_f)$, $P_F(\tau)=\prod_{i=0}^{n-1}P_F(s_{i}, s_{i+1})$. \\
$P_B$ & Backward policy, assigning probabilities to backward transitions or trajectories; for $\tau=(s_0, \dots , s_f)$, $P_B(\tau \mid s_n) = \prod_{i=0}^{n-1} P_B(s_i, s_{i+1})$. \\
$R(x), R(\tau)$ & Reward function over a terminal state $x$ or a trajectory $\tau$; $R(\tau)=R(x)P_B(\tau\mid x)$. \\
$Z$ & Estimated partition function. \\
$Z^*$ & True partition function, equal to the total reward over terminal states. \\
$P_T$ & Terminal-state distribution induced by $P_{F}$. \\
$\pi_{target}$ & Ground-truth reward distribution normalized by $Z^*$. \\
\bottomrule
\end{tabular}
\caption{Notation and Definitions}
\label{tab:notation}
\end{table}

\section{Proof of Propositions and Theorems}
\label{app:proof_of_propositions}
\textbf{Proposition 3.3.} \textit{(TV bound for incremental mode coverage over arbitrary state graph). Let the reward added at each state in a subset $\mathcal{X}_{sub}$ be $R' (x)$. We define the local true partition function $Z^*_{\mathcal{Y}\subseteq \mathcal{X}} = \sum_{x \in \mathcal{Y}} R(x)$ and the local contrast ratio $\Lambda_\mathcal{Y} = \frac{Z^*_\mathcal{Y}}{Z^*_\mathcal{Y} + \sum_{x\in\mathcal{Y}} R'(x)}$. Then, we have} 
\begin{equation*}
   \frac{Z^* - Z^*_{\mathcal{X}_{sub}}}{Z^*} (1-\Lambda_{\mathcal{X}}) \leq \mathrm{TV}(P_T, \pi_{target}) \leq (1-\Lambda_{\mathcal{X}})
\end{equation*}

\begin{proof}
    To prove the upper bound, we use the standard definition of the Total Variation distance: 
    \begin{equation*}
        \mathrm{TV}(P, Q) = \frac{1}{2}\sum_{x}\lvert P(x) - Q(x)\rvert
    \end{equation*}

    Substitute the expressions for the policies:
    \begin{equation*}
        P_T(x) - \pi_{target}(x) = \frac{R(x)}{Z^*} - \frac{R(x)+R'(x)}{Z^*+\sum_{x\in\mathcal{X}}R'(x)}
    \end{equation*}

    Substitute $Z^*+\sum_{x\in\mathcal{X}}R'(x) = Z^*/\Lambda_\mathcal{X}$:
    \begin{equation*}
        P_T(x) - \pi_{target}(x) = \frac{(1-\Lambda_\mathcal{X})R(x) - \Lambda_\mathcal{X}R'(x)}{Z^*}
    \end{equation*}

    Plug this back, we have \begin{align*}
    2\mathrm{TV}(P_T, \pi_{target}) &= \sum_{x\in\mathcal{X}} \left\lvert \frac{(1-\Lambda_\mathcal{X})R(x)}{Z^*} - \frac{ \Lambda_\mathcal{X}R'(x)}{Z^*}\right\rvert\\
    & \leq \sum_{x\in\mathcal{X}} \frac{(1-\Lambda_\mathcal{X})R(x)}{Z^*} +\sum_{x\in\mathcal{X}} \frac{ \Lambda_\mathcal{X}R'(x)}{Z^*}\\
    &= (1-\Lambda_\mathcal{X})+\frac{\sum_{x\in\mathcal{X}}R'(x)}{Z^*+\sum_{x\in\mathcal{X}}R'(x)} \\
    &=2(1-\Lambda_\mathcal{X})
    \end{align*}

    Divide both sides by 2, we get:
    \begin{equation*}
        \mathrm{TV}(P_T, \pi_{target}) \leq (1-\Lambda_\mathcal{X})
    \end{equation*}

    To prove the lower bound, we use the alternative definition of TV distance:\begin{equation*}
    \mathrm{TV}(P,Q) = \sup_{X\subseteq\mathcal{X}}\lvert P(X) - Q(X)\rvert
    \end{equation*}

    This means the difference in probability mass on any specific subset $X$ is a lower bound for the TV distance.

    Let us choose the subset to be the unchanged region $\mathcal{X}_{unc} = \mathcal{X}\setminus \mathcal{X}_{sub}$. For any state $x\in\mathcal{X}_{unc}$, $R'(x) = 0$.

    We have \begin{equation*}
        P_T(\mathcal{X}_{unc}) = \frac{\sum_{x\notin \mathcal{X}_{sub}}R(x)}{Z^*} = \frac{Z^* - Z^*_{\mathcal{X}_{sub}}}{Z^*}
    \end{equation*}
    \begin{align*}
        \pi_{target}(\mathcal{X}_{unc}) &= \frac{Z^* - Z^*_{\mathcal{X}_{sub}}}{Z^*+\sum_{x\in\mathcal{X}} R'(x)} \\&=\frac{\Lambda_{\mathcal{X}}(Z^* - Z^*_{\mathcal{X}_{sub}})}{Z^*}
    \end{align*}
    \begin{align*}
      \mathrm{TV}(P_T, \pi_{target}) & \geq \lvert P_T(\mathcal{X}_{unc}) - \pi_{target}(\mathcal{X}_{unc})\rvert \\  
     &=\frac{(Z^* - Z^*_{\mathcal{X}_{sub}}) (1-\Lambda_{\mathcal{X}})}{Z^*} 
    \end{align*}

\end{proof}

\textbf{Proposition 3.4.} \textit{(Loss scale via local contrast). While the TV error depends on the aggregate $\Lambda_{\mathcal{X}}$, the supremum of the training loss is governed by the \textbf{worst-case} local contrast ratio, given by:
\begin{equation*}
\sup  \mathcal{L}_{GFN} = \left(\log \min_{\{x\}\subseteq \mathcal{X}_{sub}} \Lambda_{\{x\}}\right) ^ 2
\end{equation*}
where $GFN \in \{FM, DB, TB, subTB\}$.}


\begin{proof}
    We assume the network parameters are currently at the optimum for the previous task. This means the current flow satisfies the balance equation for the old reward. Since only the rewards change and $P_B$ is unchanged, substituting the pre-update state into the loss functions of the new task yields that all non-zero losses satisfy
    \begin{align*}
    \mathcal{L}_{FM} (x) &= \mathcal{L}_{DB} (s,  x) =
    \mathcal{L}_{TB} (\tau \ni x) \\ 
    &=\mathcal{L}_{subTB}(\tau_{t_1:t_2}\ni x) =  \left(\log\frac{R(x)}{R(x) + R'(x)}\right)^2 \\
    &= \left(\log \Lambda_{\{x\}}\right) ^ 2
    \end{align*}

The supremum of the loss over the entire state space is determined by the state $x$ that maximizes this squared log term. Since $\Lambda_{\{x\}} \leq 1$, we have 
$$
\sup \lvert \mathcal{L}_{GFN} \rvert = \left(\log \min_{\{x\}\subseteq \mathcal{X}_{sub}} \Lambda_{\{x\}}\right) ^ 2
$$
\end{proof}

\textbf{Theorem 3.5.} \textit{(Training loss to TV distance bound). The relationship between the training loss bound and the resulting TV distance depends on the scope of the objective (trajectory-level vs. transition-level):}

\textbf{Trajectory-level Objective.} \textit{If the trajectory loss is bounded by $\mathcal{L}_{TB}(\tau) \leq c^2, \forall \tau \in \mathcal{T}$, the TV error is bounded by:
\begin{equation*}
\mathrm{TV}(P_T, \pi_{target}) \leq 1 - e^{-2c}
\end{equation*}
This bound is independent of the trajectory length, as TB optimizes the full path consistency directly.
}

\textbf{Transition-level Objective.} \textit{If the local transition loss is bounded, i.e., $\mathcal{L}_{DB}(s, s') \leq c^2$ or $\mathcal{L}_{FM}(s') \leq c^2$, the global consistency relies on the accumulation of local estimates. For trajectories of maximum number of transitions $L$, the error bound degrades linearly with depth in the log-domain:
\begin{equation*}
\mathrm{TV}(P_T, \pi_{target}) \leq 1 - e^{-2Lc}
\end{equation*}
}

\begin{proof}
We first prove the trajectory-level bound and then derive the DB and FM cases by bounding the corresponding trajectory log-ratios.

Let $Z$ be the learnable partition function. When TB loss is bounded by $c^2$, then $\forall \tau$ with $x \in \tau$ and $x \in \mathcal{X}$, we have:
\begin{equation*}
    e^{-c} \leq \frac{Z P_F(\tau)}{R(x)P_B(\tau\mid x)}\leq e^c
\end{equation*}
Aggregate this inequality over all possible trajectories in the entire space $\mathcal{T}$, we have
\begin{align*}
    e^{-c} &\leq \frac{Z}{Z^*}\leq e^c \\
    Z^*e^{-c} &\leq Z \leq Z^*e^c
\end{align*}

Also from the TB definition, we have \begin{align*}
    P_F(\tau) &\geq e^{-c} \frac{R(x)P_B(\tau\mid x)}{Z} \\
     &\geq e^{-c} \frac{R(x)P_B(\tau\mid x)}{e^cZ^*} \\
     &\geq e^{-2c}\frac{R(x)P_B(\tau\mid x)}{Z^*}
\end{align*}

Sum over all trajectories that lead to $x$, we have \begin{align*}
P_T(x)\geq e^{-2c}\frac{R(x)}{Z^*} = e^{-2c}\pi_{target}(x)
\end{align*}

We use another alternative definition of TV distance:\begin{equation*}
\mathrm{TV}(P,Q) = 1 - \sum_{x\in\mathcal{X}}\min(P, Q)
\end{equation*}

If $P_T(x) \geq \pi_{target}(x)$, then $\min(P_T(x), \pi_{target}(x)) = \pi_{target}(x)$. Since $e^{-2c} \leq 1$, $\pi_{target}(x) \geq e^{-2c} \pi_{target}(x)$. If $P_T(x) < \pi_{target}(x)$, then $\min(P_T(x), \pi_{target}(x)) = P_T(x) \geq e^{-2c} \pi_{target}(x)$.

In both cases, we have \begin{equation*}
    \min(P_T, \pi_{target}) \geq e^{-2c}\pi_{target}(x)
\end{equation*}

Substitute into the TV equation:
\begin{align*}
    \mathrm{TV}(P_T, \pi_{target}) &\leq 1 - \sum_{x\in\mathcal{X}} e^{-2c}\pi_{target(x)}\\
    & = 1 - e^{-2c}
\end{align*}

\paragraph{Extension to DB and FM.}
Consider a complete trajectory
$\tau=(s_0\to s_1\to\cdots\to s_n=x\to s_f)$, where $n\leq L$.
Under our convention, $P_F(x,s_f)=1$, so
$$
P_F(\tau)=\prod_{t=0}^{n-1}P_F(s_t,s_{t+1}),
\qquad
P_B(\tau\mid x)=\prod_{t=0}^{n-1}P_B(s_t,s_{t+1}).
$$

In the rest of the proof, $F(s)$ and $F(s,s')$ denote the
positive estimates used in the respective training objectives; exact
flow conservation is not assumed.

\paragraph{Detailed balance.}
Take $Z=F(s_0)$ and impose $F(x)=R(x)$ at each terminal state.
Define the signed DB residual
$$
r_{\mathrm{DB}}(s,s')
:=\log\frac{F(s)P_F(s,s')}{F(s')P_B(s,s')}.
$$
Because $\mathcal{L}_{\mathrm{DB}}(s,s')
=r_{\mathrm{DB}}(s,s')^2\leq c^2$, we have
$|r_{\mathrm{DB}}(s,s')|\leq c$.
The state-flow terms telescope along any trajectory:
\begin{align*}
\sum_{t=0}^{n-1}r_{\mathrm{DB}}(s_t,s_{t+1})
&=\log\prod_{t=0}^{n-1}
\frac{F(s_t)P_F(s_t,s_{t+1})}
     {F(s_{t+1})P_B(s_t,s_{t+1})}\\
&=\log\frac{F(s_0)P_F(\tau)}{F(s_n)P_B(\tau\mid x)}\\
&=\log\frac{ZP_F(\tau)}{R(x)P_B(\tau\mid x)}.
\end{align*}
Consequently,
$$
\left|\log\frac{ZP_F(\tau)}{R(x)P_B(\tau\mid x)}\right|
\leq\sum_{t=0}^{n-1}|r_{\mathrm{DB}}(s_t,s_{t+1})|
\leq nc\leq Lc.
$$
Thus $\mathcal{L}_{\mathrm{TB}}(\tau)\leq L^2c^2$ for every
$\tau\in\mathcal{T}$.

\paragraph{Flow matching.}
Let $F(s,s')>0$ be the learned edge-flow estimates on transitions with
$s'\neq s_f$, and fix $F(x,s_f)=R(x)$ for $x\in\mathcal{X}$.
Define the incoming and outgoing flow sums by
$$
I(s):=\sum_{u:\,u\to s\in\mathcal{A}}F(u,s),
\qquad s\notin\{s_0,s_f\},
$$
$$
O(s):=R(s)+
\sum_{v:\,s\to v\in\mathcal{A},\ v\neq s_f}F(s,v),
\qquad s\neq s_f.
$$
In particular, $O(x)=R(x)$ for each terminal state.
Set $Z=O(s_0)$ and define, for $s\to s'\in\mathcal{A}$ with
$s'\neq s_f$,
$$
P_F(s,s')=\frac{F(s,s')}{O(s)},
\qquad P_B(s,s')=\frac{F(s,s')}{I(s')}.
$$
Together with $P_F(x,s_f)=1$, these are normalized forward and backward
policies. The backward policy $P_B$ is an auxiliary
construction for the proof; it need not be separately learned.
Importantly, $I(s)$ and $O(s)$ are not assumed equal.

The signed FM residual is
$$
r_{\mathrm{FM}}(s):=\log\frac{I(s)}{O(s)},
\qquad
\mathcal{L}_{\mathrm{FM}}(s)=r_{\mathrm{FM}}(s)^2.
$$
For each nonfinal transition, the common edge-flow estimate cancels:
$$
\frac{P_F(s,s')}{ P_B(s,s')}
=\frac{I(s')}{O(s)}.
$$
Therefore, for any complete trajectory ending at $x=s_n$,
\begin{align*}
\frac{ZP_F(\tau)}{R(x) P_B(\tau\mid x)}
&=\frac{O(s_0)}{O(s_n)}
  \prod_{t=0}^{n-1}\frac{I(s_{t+1})}{O(s_t)}\\
&=\prod_{t=1}^{n}\frac{I(s_t)}{O(s_t)}.
\end{align*}
Taking logarithms gives the required telescoping identity:
$$
\log\frac{ZP_F(\tau)}{R(x) P_B(\tau\mid x)}
=\sum_{t=1}^{n}r_{\mathrm{FM}}(s_t).
$$
Since the FM bound applies to every visited noninitial state, including
$s_n=x$, it follows that
$$
\left|\log\frac{ZP_F(\tau)}{R(x) P_B(\tau\mid x)}\right|
\leq\sum_{t=1}^{n}|r_{\mathrm{FM}}(s_t)|
\leq nc\leq Lc.
$$
Hence the corresponding TB loss, evaluated with $Z=O(s_0)$ and
$ P_B$, is at most $L^2c^2$ on every trajectory.

\end{proof}

\textbf{Theorem 3.6.} \textit{(Probabilistic TV bound via trajectory sampling). Given the ground-truth target distribution $\pi_{target}$, we define a target distribution over trajectories $\hat{\pi}(\tau) = \pi_{target}(x_\tau) P_B(\tau | x_\tau)$. Sample $m$ trajectories $\tau_1, \dots, \tau_m$ from $\hat{\pi}$ independently by sampling $x \sim \pi_{target}$ and $\tau \sim P_B(\cdot|x)$. Sample another $n$ trajectories independently using $P_F$. Let $c = \max_{i\leq m+n} \sqrt{\mathcal{L}_{TB}(\tau_i)}$, with confidence $1-2\alpha$, the global TV error is bounded by:}
\begin{equation*}
\mathrm{TV}(P_T, \pi_{target}) \leq e^{2c} + 1 - \alpha^\frac{1}{m} -  \alpha^\frac{1}{n} \leq e^{2c} - 1 + \frac{\log(1/\alpha)}{m} +\frac{\log(1/\alpha)}{n}\end{equation*}

\begin{proof}
Let $P_F(\tau)$ be the forward policy's probability of trajectory $\tau$. Let $\hat{\pi}(\tau)$ be the target trajectory distribution induced by the reward and backward policy:
$$\hat{\pi}(\tau) = \frac{R(x)}{Z^*} P_B(\tau | x) = \pi_{target}(x) P_B(\tau | x)$$

Note that the marginal of $\hat{\pi}$ over states is exactly the target distribution: $\sum_{\tau \to x} \hat{\pi}(\tau) = \pi_{target}(x)$.

Define the ``good'' set of trajectories $G_\mathcal{T} = \{ \tau \in \mathcal{T} \mid \mathcal{L}_{TB}(\tau) \leq c^2 \}$ and its probability mass measured under target trajectory distribution $\hat{\pi}(G_\mathcal{T})$ as $1-\epsilon_c$.

The one-sided distribution-free tolerance bound~\citep{hahn2011statistical} gives, for any $\xi \in (0,1)$ 
$$
Pr(\epsilon_c > \xi) \leq (1-\xi)^m
$$

Setting $\xi = 1-\alpha^{1/m}$, we get: with probability at least $1-\alpha$, $$
\epsilon_c \leq 1-\alpha^{1/m} = 1 - e^{-\log(1/\alpha)/m} \leq \frac{\log(1/\alpha)}{m}
$$

Similarly, let $P_F(G_{\mathcal{T}}) = 1-\eta_c$, since we observed $n$ independent samples from $P_F$ and all fell into $G_\mathcal{T}$, we get: with confidence $1-\alpha$:

$$\eta_c \leq 1-\alpha^{1/n} \leq \frac{\log(1/\alpha)}{n}
$$

By the union bound, we have $\epsilon_c \leq 1-\alpha^{1/m}$ and 
$\eta_c \leq 1-\alpha^{1/n}$ together hold with confidence $1-2\alpha$.

For any good trajectory $\tau \in G_\mathcal{T}$, the condition $\mathcal{L}_{TB}(\tau) \leq c^2$ implies:$$\left| \log \frac{Z P_F(\tau)}{R(x) P_B(\tau|x)} \right| \leq c $$

$$e^{-c} \leq \frac{KP_F(\tau)}{\hat{\pi}(\tau)} \leq e^c$$

where $K = \frac{Z}{Z^*}$.

Now we derive the bound for $K$, note 
$$
K\sum_{\tau \in G_{\mathcal{T}}}P_F(\tau) \geq \sum_{\tau \in G_{\mathcal{T}}} e^{-c}\hat{\pi}(\tau) = e^{-c}(1-\epsilon_c)
$$

Since $\sum_{\tau \in G_{\mathcal{T}}}P_F(\tau) \leq 1$, we have $$
K\geq e^{-c} (1 - \epsilon_c) $$


So $\frac{P_F(\tau)}{\hat{\pi}(\tau)}$ satisfies
$$
\frac{P_F(\tau)}{\hat{\pi}(\tau)} \leq \frac{1}{1-\epsilon_c}e^{2c}
$$
 
The TV distance between the two distributions over trajectories is bounded by:

\begin{align*}
    2\mathrm{TV}(P_F, \hat{\pi}) &= \sum_{\tau \in G_{\mathcal{T}}} \lvert P_F(\tau) - \hat{\pi}(\tau)\rvert + \sum_{\tau \in G^{\complement}_{\mathcal{T}}}\lvert P_F(\tau) - \hat{\pi}(\tau)\rvert 
\end{align*}

For the first term $\sum_{\tau \in G_{\mathcal{T}}} \lvert P_F(\tau) - \hat{\pi}(\tau)\rvert$, we split the good set $G_{\mathcal{T}}$ into $G^+_{\mathcal{T}}$ and $G^-_{\mathcal{T}}$ such that

$$
P_F(\tau) \geq \hat{\pi}(\tau), \forall \tau \in G^+_{\mathcal{T}}
$$

$$
P_F(\tau) < \hat{\pi}(\tau), \forall \tau \in G^-_{\mathcal{T}}
$$

We have 
\begin{align*}
&\sum_{\tau \in G_{\mathcal{T}}} \lvert P_F(\tau) - \hat{\pi}(\tau)\rvert \\= &\sum_{\tau \in G^+_{\mathcal{T}}} ( P_F(\tau) - \hat{\pi}(\tau)) + \sum_{\tau \in G^-_{\mathcal{T}}} ( \hat{\pi}(\tau) - P_F(\tau))
\\
=&\sum_{\tau \in G^+_{\mathcal{T}}} ( P_F(\tau) - \hat{\pi}(\tau)) + (1-\epsilon_c)-\sum_{\tau \in G^+_{\mathcal{T}}} \hat{\pi}(\tau)\\
&-(1-\eta_c) + \sum_{\tau \in G^+_{\mathcal{T}}}  P_F(\tau)\\
 =& 2\sum_{\tau \in G^+_{\mathcal{T}}} ( P_F(\tau) - \hat{\pi}(\tau)) + \eta_c - \epsilon_c\\
 =& 2\sum_{\tau \in G^+_{\mathcal{T}}} \hat{\pi}(\tau)(\frac{P_F(\tau)}{\hat{\pi}(\tau)} - 1) + \eta_c - \epsilon_c\\
 \leq & 2 (1-\epsilon_c)(\frac{1}{1-\epsilon_c}e^{2c}-1) + \eta_c - \epsilon_c\\
 = & 2(e^{2c} - 1) + \eta_c + \epsilon_c 
\end{align*}

For the second term $\sum_{\tau \in G^{\complement}_{\mathcal{T}}}\lvert P_F(\tau) - \hat{\pi}(\tau)\rvert $, we directly use the triangle inequality and get 

\begin{align*}
\sum_{\tau \in G^{\complement}_{\mathcal{T}}}\lvert P_F(\tau) - \hat{\pi}(\tau)\rvert  &\leq \sum_{\tau \in G^{\complement}_{\mathcal{T}}} P_F(\tau)+ \sum_{\tau \in G^{\complement}_{\mathcal{T}}} \hat{\pi}(\tau)\\
& = \eta_c + \epsilon_c
\end{align*}

Combine the first term and the second term, we have $$
\mathrm{TV}(P_F, \hat{\pi}) \leq (e^{2c}-1)+\epsilon_c + \eta_c
$$

For the TV distance between marginal distributions, we have 

$$\mathrm{TV}(P_T, \pi_{target}) = \frac{1}{2} \sum_x | P_T(x) - \pi_{target}(x)|$$

Since $P_T(x) = \sum_{\tau \to x} P_F(\tau)$ and $\pi_{target}(x) = \sum_{\tau \to x} \hat{\pi}(\tau)$:
\begin{align*}
| P_T(x) - \pi_{target}(x) | &= \left| \sum_{\tau \to x} (P_F(\tau) - \hat{\pi}(\tau)) \right| \\&\leq \sum_{\tau \to x} | P_F(\tau) - \hat{\pi}(\tau) |
\end{align*}

Summing over all $x$:

\begin{align*}
\mathrm{TV}(P_T, \pi_{target}) &\leq \frac{1}{2} \sum_x \sum_{\tau \to x} | P_F(\tau) - \hat{\pi}(\tau) | \\&= \mathrm{TV}(P_F, \hat{\pi})\\
&\leq (e^{2c}-1)+\epsilon_c + \eta_c\\
&\leq (e^{2c}-1)+ 1 - \alpha^{1/m} + 1- \alpha^{1/n}\\
&= e^{2c}+ 1 - \alpha^{1/m} - \alpha^{1/n}\\
&\leq (e^{2c}-1)+ \frac{\log(1/\alpha)}{m} + \frac{\log(1/\alpha)}{n}
\end{align*}

\end{proof}

\textbf{Corollary 3.7.} \textit{(Subgraph Certification via trajectory sampling).
Let $\mathcal{X}_{sub} \subseteq \mathcal{X}$ be a subset of end-states. Define the restricted target distribution over $\mathcal{X}_{sub}$ by $\pi_{target}^{sub}(x) = R(x)/\sum_{x\in\mathcal{X}_{sub}}R(x)$, the corresponding restricted target trajectory distribution as $\hat{\pi}_{sub}(\tau) = \pi^{sub}_{target}(x_\tau) P_B(\tau | x_\tau)$. Sample $m$ trajectories $\tau_1, \dots, \tau_m$ from $\hat{\pi}_{sub}$. Sample another $n$ trajectories that end within $\mathcal{X}_{sub}$ independently using $P_F$. Suppose we observe $\mathcal{L}_{TB}(\tau_i) \leq c^2$ for all trajectories in both sets. Let $P_T^{\mathrm{sub}}$ denote the terminal flow $P_T$ renormalized to $\mathcal{X}_{\mathrm{sub}}$.
Then, with confidence $1-2\alpha$,
\begin{equation*}
\mathrm{TV}(P_T^{sub}, \pi^{sub}_{target}) \leq e^{2c} + 1 - \alpha^\frac{1}{m} -  \alpha^\frac{1}{n} \leq e^{2c} - 1 + \frac{\log(1/\alpha)}{m} +\frac{\log(1/\alpha)}{n}\end{equation*}}

\begin{proof}
To prove this corollary, we first develop a mapping from the global forward policy $P_F$ and partition function $Z$ to the subgraph restricted counterparts $P_F^{sub}$ and $Z_{sub}$.

Consider the subgraph induced by $\mathcal{X}_{sub}$. We have the partition function for the subgraph 
$$Z_{sub} =\sum_{x\in \mathcal{X}_{sub}}F(x)= Z \sum_{x\in \mathcal{X}_{sub}}P_T(x)$$

The backward policy over the subgraph yields the same value as the global one, i.e., $P_B^{sub} = P_B$. The forward policy over this subgraph $P_F^{sub}$ can be represented as the forward policy conditioned on the trajectory ending in the subgraph:
$$
    P_F^{sub}(\tau) = P_F(\tau \mid x_\tau \in \mathcal{X}_{sub}) = \frac{P_F(\tau)}{\sum_{x\in \mathcal{X}_{sub}}P_T(x)}
$$

Now we map the TB loss to the subgraph context, we have the TB loss over the subgraph become:
\begin{align*}
    \mathcal{L}^{sub}_{TB}(\tau) &=\left( \log \frac{Z_{sub} P^{sub}_F(\tau)}{R(x) P^{sub}_B(\tau|x)} \right)^2\\
    &= \left( \log \frac{Z P_F(\tau)\sum_{x\in \mathcal{X}_{sub}}P_T(x)}{R(x) P_B(\tau|x)\sum_{x\in \mathcal{X}_{sub}}P_T(x)} \right)^2\\ &= \left( \log \frac{Z P_F(\tau)}{R(x) P_B(\tau|x)} \right)^2 \leq c^2
\end{align*}

Since all preconditions of Theorem~\ref{thm:bigtheorem} are satisfied for the subgraph, we apply it directly to obtain the desired bound.


\end{proof}

\textbf{Theorem 3.10.} \textit{(Fidelity trade-off under reference flow).
The fidelity of the recovered policy depends on the ratio between the training loss and the augmentation magnitude. Let the total reference flow be $\Delta = \sum_{\tau \in \mathcal{T}} \delta(\tau)$. If the reference training loss is bounded by $\mathcal{L}_{aug} (\tau) \leq c^2$, the terminal distribution $P_T$ induced by the learned forward policy satisfies:
\begin{equation*}
\mathrm{TV}(P_{T}, \pi_{target}) \leq \frac{(1 - e^{-2c})(1+\Delta/Z^*)}{1 + (1-e^{-c})\Delta/Z^*} \leq  (1 - e^{-2c})(1+\frac{\Delta}{Z^*}) \end{equation*}
}

\begin{proof}
Define $\delta(x) = \sum_{\tau \ni x} \delta(\tau)$, note $\Delta = \sum_{x\in\mathcal{X}} \delta(x)$. We have the target distribution under reference flows as:
\begin{equation*}
    \pi_{aug}(x) = \frac{R(x)+\delta(x)}{Z^* + \Delta}
\end{equation*}

From $\mathcal{L}_{\mathrm{aug}}(\tau)\leq c^2$

$$
e^{-c}\left(R(\tau)+\delta(\tau)\right)\leq F(\tau)+\delta(\tau)\leq e^c\left(R(\tau)+\delta(\tau)\right).
$$

Summing the upper inequality over all trajectories gives

$$
Z+\Delta\leq e^c\left(Z^*+\Delta\right)
$$

and therefore

$$
Z\leq Z^*\left[e^c+\left(e^c-1\right)\frac{\Delta}{Z^*}\right]
$$

Summing the lower inequality over trajectories terminal at $x$ gives

$$
F(x)+\delta(x)\geq e^{-c}\left(R(x)+\delta(x)\right)
$$

Since $F(x)\geq 0$,

$$
F(x)\geq R(x)\left[e^{-c}-\left(1-e^{-c}\right)\frac{\delta(x)}{R(x)}\right]_{+}
$$

From $P_T(x) := \frac{F(x)}{Z}$,

$$
P_T(x)\geq \pi_{target}(x)\frac{\left[e^{-c}-\left(1-e^{-c}\right)\delta(x)/R(x)\right]_{+}}{e^c+\left(e^c-1\right)\Delta/{Z^*}}
$$

Because $\frac{\left[e^{-c}-\left(1-e^{-c}\right)\delta(x)/R(x)\right]_{+}}{e^c+\left(e^c-1\right)\Delta/{Z^*}}\leq 1$,

$$
\min\left(P_T(x),\pi_{target}(x)\right)\geq \pi_{target}(x)\frac{\left[e^{-c}-\left(1-e^{-c}\right)\delta(x)/R(x)\right]_{+}}{e^c+\left(e^c-1\right)\Delta/{Z^*}}
$$

Summing over $x$ and applying Jensen's inequality to the convex positive-part function gives

$$
\sum_{x\in\mathcal{X}}\min\left(P_T(x),\pi_{target}(x)\right)\geq\frac{\left[e^{-c}-\left(1-e^{-c}\right)\sum_{x\in\mathcal{X}}\pi_{target}(x)\delta(x)/R(x)\right]_{+}}{e^c+\left(e^c-1\right)\Delta/Z^*}
$$

Moreover,

$$
\sum_{x\in\mathcal{X}}\pi_{target}(x)\frac{\delta(x)}{R(x)}=\sum_{x\in\mathcal{X}}\frac{R(x)}{Z^*}\frac{\delta(x)}{R(x)}=\frac{1}{Z^*}\sum_{x\in\mathcal{X}}\delta(x)=\frac{\Delta}{Z^*}
$$

Therefore,

$$
\sum_{x\in\mathcal{X}}\min\left(P_T(x),\pi_{target}(x)\right)\geq\frac{\left[e^{-c}-\left(1-e^{-c}\right)\Delta/Z^*\right]_{+}}{e^c+\left(e^c-1\right)\Delta/Z^*}
$$

Using

$$
\mathrm{TV}\left(P_T,\pi_{target}\right)=1-\sum_{x\in\mathcal{X}}\min\left(P_T(x),\pi_{target}(x)\right)
$$

we get

$$
\mathrm{TV}\left(P_T,\pi_{target}\right)\leq \frac{\left(1-e^{-2c}\right)\left(1+\Delta/Z^*\right)}{1+\left(1-e^{-c}\right)\Delta/Z^*}\leq \left(1-e^{-2c}\right)\left(1+\frac{\Delta}{Z^*}\right)
$$

\end{proof}

\textbf{Theorem 3.11.} \textit{(Probabilistic TV bound with optimizable reference-flow threshold).
Sample $m(\geq2)$ trajectories $\tau_1, \dots, \tau_m$ independently from the target $\hat{\pi}$, sample another $n(\geq2)$ trajectories  $\tau_{m+1}, \dots, \tau_{m+n}$ independently using $P_F$. For each $c>0$, we compute the minimum reference flow $\delta_c(\tau_i)$ according to Equation~ (\ref{equ:reference_flow}). Define $M_c:=\max_{i\in\{1,\dots,m+n\}}\frac{\delta_c(\tau_i)}{R(\tau_i)}$, and $\mathcal{C} = \{c>0\mid M_c < \frac{1}{e^c-1}\}$. With confidence $1-2\alpha$, the following bound holds simultaneously for every $c\in\mathcal{C}$:
\begin{align*}
\label{equ:probablistic_TV_with_reference_flow}
\mathrm{TV}(P_{T}, \pi_{target}) &\leq \beta_m(\alpha) + \beta_n(\alpha)+ \left(\frac{e^c + (e^c-1)M_c}{e^{-c} - (1-e^{-c}) M_c} - 1\right)\\ & \leq \frac{2\log(2/\alpha)}{m} + \frac{2\log(2/\alpha)}{n} \nonumber+ \left(\frac{e^c + (e^c-1)M_c}{e^{-c} - (1-e^{-c}) M_c} - 1\right)
\end{align*}
where $\beta_k(\alpha)$ is solved from \begin{equation*}
    (1 - \beta_k(\alpha))^{k-1}[1+(k-1)\beta_k(\alpha)] = \alpha
\end{equation*}
}

\begin{proof}




Define $\rho(\tau) = \frac{ZP_F(\tau)}{R(\tau)}$, and the ``$c$-dependent'' good set of trajectories $G_{\mathcal{T}} \subseteq \mathcal{T}$ as those satisfying

\begin{equation*}
    G_{\mathcal{T}} = \{\tau\in \mathcal{T}\mid \frac{\delta_c(\tau)}{R(\tau)} \leq M_c\}
\end{equation*}

Equation ($\ref{equ:reference_flow}$) implies the good set $$
G_{\mathcal{T}} = \{\tau\in \mathcal{T}\mid \frac{\delta_c(\tau)}{R(\tau)} \leq M_c\} = \{\tau\in \mathcal{T}\mid L_c \leq \rho(\tau) \leq U_c\}
$$

where $L_c = e^{-c} - (1 -e^{-c})M_c$ and $U_c = e^{c} + (e^c -1)M_c$.

Next, we show that for every $c\in \mathcal{C}$, every $c$-dependent good set contains the empirical range of the samples. To see this, let $\rho_{min} = \min_{i\leq m +n} \frac{ZP_F(\tau_i)}{R(\tau_i)}$ and $\rho_{max} = \max_{i\leq m +n} \frac{ZP_F(\tau_i)}{R(\tau_i)}$. We then have $$
[\rho_{min},\rho_{max}] \subseteq [L_c, U_c]
$$

Hence, every $c$-dependent bad set (i.e., complementary to the good set), whether evaluated forward or backward, lies in the corresponding empirical-range complement, which all share a single order-statistic bound. Let $\epsilon_c$ be the probability mass of the bad set under the target trajectory distribution, $\eta_c$ be the probability mass of the bad set under the forward policy, and $\beta_k(\alpha)$ solve $$
 (1 - \beta_k(\alpha))^{k-1}[1+(k-1)\beta_k(\alpha)] = \alpha
$$

By the two-sided distribution-free tolerance bound~\citep{hahn2011statistical} and the union bound, with confidence $1-2\alpha$, \textbf{simultaneously for every} $c\in \mathcal{C}$, one has  $$
\epsilon_c \leq \beta_m(\alpha), \quad \eta_c \leq \beta_n(\alpha)
$$

A computationally simpler but looser bound follows by controlling the tails separately. Assigning failure probability $\alpha/2$ to each tail gives, with confidence $1-2\alpha$, one has

$$
\epsilon_c \leq \frac{2\log(2/\alpha)}{m}
, \quad \eta_c \leq \frac{2\log(2/\alpha)}{n}
$$


For any good trajectory $\tau \in G_{\mathcal{\tau}}$, we have $$\left| \log \frac{Z P_F(\tau)+\delta_c(\tau)}{R(x) P_B(\tau|x)+\delta_c(\tau)} \right| \leq c $$

$$
e^{-c}\leq \frac{K P_F(\tau)+\delta_c(\tau)/Z^*}{\hat{\pi}(\tau) + \delta_c(\tau)/Z^*}\leq e^c
$$

where $K = \frac{Z}{Z^*}$

Now we derive the bound for $K$, since $\sum_{\tau\in G_{\mathcal{T}}} P_F(\tau)\leq 1$, 

\begin{align*}
K &\geq \sum_{\tau \in G_{\mathcal{T}}}e^{-c} 
\hat{\pi}(\tau)- (1-e^{-c}) \sum_{\tau \in G_\mathcal{T}}\frac{\delta_c(\tau)}{Z^*} \\
& = e^{-c}(1-\epsilon_c) - (1-e^{-c}) \sum_{\tau \in G_\mathcal{T}}\frac{R(\tau)}{Z^*} \frac{\delta_c(\tau)}{R(\tau)}\\
& \geq e^{-c}(1-\epsilon_c) - (1-e^{-c}) M_c\sum_{\tau \in G_\mathcal{T}}\frac{R(\tau)}{Z^*} \\
&\geq  e^{-c}(1-\epsilon_c) - (1-e^{-c}) (1-\epsilon_c) M_c
\end{align*}

Since $M_c < \frac{1}{e^c-1} $, we have $$
(1-\epsilon_c)\left(e^{-c} - (1-e^{-c}) M_c \right)> 0
$$

Plug this back, we get $$
P_F(\tau) \leq \frac{e^c \hat{\pi}(\tau) + (e^c-1)\delta_c(\tau)/Z^*}{(1-\epsilon_c)\left(e^{-c} - (1-e^{-c}) M_c \right)}
$$
\begin{align*}
\frac{P_F(\tau)}{\hat{\pi}(\tau)} &\leq \frac{e^c + (e^c-1)\delta_c(\tau)/R(\tau)}{(1-\epsilon_c)\left(e^{-c} - (1-e^{-c}) M_c \right)}\\
&\leq \frac{e^c + (e^c-1)M_c}{(1-\epsilon_c)\left(e^{-c} - (1-e^{-c}) M_c \right)}
\end{align*}

The TV distance between the two distributions over trajectories is bounded by:

\begin{align*}
    2\mathrm{TV}(P_F, \hat{\pi}) &= \sum_{\tau \in G_{\mathcal{T}}} \lvert P_F(\tau) - \hat{\pi}(\tau)\rvert + \\ &\sum_{\tau \in G^{\complement}_{\mathcal{T}}}\lvert P_F(\tau) - \hat{\pi}(\tau)\rvert 
\end{align*}

For the first term $\sum_{\tau \in G_{\mathcal{T}}} \lvert P_F(\tau) - \hat{\pi}(\tau)\rvert$, we split the good set $G_{\mathcal{T}}$ into $G^+_{\mathcal{T}}$ and $G^-_{\mathcal{T}}$ such that

$$
P_F(\tau) \geq \hat{\pi}(\tau), \forall \tau \in G^+_{\mathcal{T}}
$$

$$
P_F(\tau) < \hat{\pi}(\tau), \forall \tau \in G^-_{\mathcal{T}}
$$

We have 
\begin{align*}
&\sum_{\tau \in G_{\mathcal{T}}} \lvert P_F(\tau) - \hat{\pi}(\tau)\rvert \\= &\sum_{\tau \in G^+_{\mathcal{T}}} ( P_F(\tau) - \hat{\pi}(\tau)) + \sum_{\tau \in G^-_{\mathcal{T}}} ( \hat{\pi}(\tau) - P_F(\tau))
\\
=&\sum_{\tau \in G^+_{\mathcal{T}}} ( P_F(\tau) - \hat{\pi}(\tau)) + (1-\epsilon_c)-\sum_{\tau \in G^+_{\mathcal{T}}} \hat{\pi}(\tau)-(1-\eta_c) + \sum_{\tau \in G^+_{\mathcal{T}}}  P_F(\tau)\\
 =& 2\sum_{\tau \in G^+_{\mathcal{T}}} ( P_F(\tau) - \hat{\pi}(\tau)) + \eta_c - \epsilon_c\\
 =& 2\sum_{\tau \in G^+_{\mathcal{T}}} \hat{\pi}(\tau)(\frac{P_F(\tau)}{\hat{\pi}(\tau)} - 1) + \eta_c - \epsilon_c\\
 \leq & 2 (1-\epsilon_c)( \frac{e^c + (e^c-1)M_c}{(1-\epsilon_c)\left(e^{-c} - (1-e^{-c}) M_c \right)}-1) + \eta_c - \epsilon_c\\
 = & 2(\frac{e^c + (e^c-1)M_c}{e^{-c} - (1-e^{-c}) M_c} - 1) + \eta_c + \epsilon_c 
\end{align*}

For the second term $\sum_{\tau \in G^{\complement}_{\mathcal{T}}}\lvert P_F(\tau) - \hat{\pi}(\tau)\rvert $, we directly use the triangle inequality and get 

\begin{align*}
\sum_{\tau \in G^{\complement}_{\mathcal{T}}}\lvert P_F(\tau) - \hat{\pi}(\tau)\rvert  &\leq \sum_{\tau \in G^{\complement}_{\mathcal{T}}} P_F(\tau)+ \sum_{\tau \in G^{\complement}_{\mathcal{T}}} \hat{\pi}(\tau)\\
& = \eta_c + \epsilon_c
\end{align*}

Combine the first term and the second term, we have $$
\mathrm{TV}(P_F, \hat{\pi}) \leq \left(\frac{e^c + (e^c-1)M_c}{e^{-c} - (1-e^{-c}) M_c} - 1\right)+\epsilon_c + \eta_c
$$

For the TV distance between marginal distributions, we have 

$$\mathrm{TV}(P_T, \pi_{target}) = \frac{1}{2} \sum_x | P_T(x) - \pi_{target}(x)|$$

Since $P_T(x) = \sum_{\tau \to x} P_F(\tau)$ and $\pi_{target}(x) = \sum_{\tau \to x} \hat{\pi}(\tau)$:
\begin{align*}
| P_T(x) - \pi_{target}(x) | &= \left| \sum_{\tau \to x} (P_F(\tau) - \hat{\pi}(\tau)) \right| \leq \sum_{\tau \to x} | P_F(\tau) - \hat{\pi}(\tau) |
\end{align*}

Summing over all $x$:

\begin{align*}
\mathrm{TV}(P_T, \pi_{target}) &\leq \frac{1}{2} \sum_x \sum_{\tau \to x} | P_F(\tau) - \hat{\pi}(\tau) | \\&= \mathrm{TV}(P_F, \hat{\pi})\\
&\leq \left(\frac{e^c + (e^c-1)M_c}{e^{-c} - (1-e^{-c}) M_c} - 1\right)+\epsilon_c + \eta_c\\
&\leq \left(\frac{e^c + (e^c-1)M_c}{e^{-c} - (1-e^{-c}) M_c} - 1\right)+  \beta_m(\alpha) + \beta_n(\alpha)\\
&\leq \left(\frac{e^c + (e^c-1)M_c}{e^{-c} - (1-e^{-c}) M_c} - 1\right)+  \frac{2\log(2/\alpha)}{m} + \frac{2\log(2/\alpha)}{n}
\end{align*}

\end{proof}

\section{Connection Between Our Theoretical Findings and Existing GFlowNets Training Approaches}
\label{app:extended_related_work}


We analyze existing advances in GFlowNet training through the lens of our theoretical results, interpreting how diverse algorithmic design choices can implicitly improve stability and certification performance.

\begin{itemize}

\item \textbf{Annealing Schedules.} \citet{kim2023learning} proposed scaling logits (temperature annealing), while \citet{chen2023order} used order-preserving flows.
While these methods do not explicitly use a reference flow variable $\delta$, they achieve a similar stabilizing effect by manipulating the target distribution. High initial temperatures flatten the energy landscape, keeping log-ratios $\log \frac{P_F}{R}$ small, thereby mitigating abrupt increases in flow mismatch and postponing large loss explosions.

\item \textbf{Sub-trajectory and Partial Losses.} \citet{madan2023learning} introduced Sub-Trajectory Balance (SubTB) to assign credit to partial trajectories, while \citet{shen2023towards} parametrized policies over transitions. By decomposing the global trajectory loss into local constraints, these methods reduce the variance of the flow mismatch. In our framework, lower variance implies that a smaller reference flow $\delta(\tau)$ is sufficient to satisfy the stability condition $\mathcal{L} \leq c^2$, directly improving the certification performance.

\item \textbf{Divergence and Distributional Objectives.} Distributional objectives~\citep{zhang2023distributional} and alternative divergence-based losses~\citep{silva2024divergence, hu2024beyond} replace mean squared flow mismatches with smoother optimization criteria. These formulations suppress extreme gradient spikes, implicitly bounding the loss and improving training stability.

\item \textbf{Guided Exploration.} GAFN~\citep{pan2022gafn} and Double GFN~\citep{lau2023dgfn} inject stochasticity through reference transitions or dual-network designs. By forcing the policies to stay random enough to keep looking for better solutions, they promote the representativeness of end-state coverage and therefore improve verification effectiveness. \citet{kim2023local} utilized backward sampling for local search around discovered modes. \citet{lau2024qgfn} combined GFlowNets with Q-functions, while \citet{ikram2024evolution}, \citet{kim2024genetic}, and \citet{kim2024ant} integrated evolutionary and genetic algorithms to guide exploration. Additionally, \citet{kim2024adaptive} leveraged a teacher policy to focus sampling specifically on regions where the student policy exhibits high loss. These methods aggressively expand the support of the discovered subgraph $\mathcal{X}_{sub}$. In the context of our bounds, they implicitly sample more high-loss/undervisited ``backward'' sampled trajectories associated with high reward. By encouraging the training process to observe and minimize loss on these specific trajectories, these methods improve the verification bound.

\item \textbf{Backward Policy Optimization.} \citet{jang2024pessimistic} proposed pessimistic backward policies. Through the lens of Theorem~\ref{thm:bigtheorem}, this biases $P_B$ to align backward samples with forward-visited trajectories. This alignment improves verification performance.

\item \textbf{Replay Buffer.} \citet{shen2023towards} introduced reward-prioritized replay, which, through the lens of Theorem~\ref{thm:bigtheorem}, \emph{biases the empirical sampling distribution} toward backward-sampled trajectories (that terminate in high-reward states) and repeatedly revisiting these trajectories helps improve the verification bounds.

\item 
\textbf{Structure and Pre-training.} \citet{nguyen2023hierarchical} and \citet{shen2023tacogfn} formulated GFlowNets for hierarchical and conditional generation, while \citet{pan2023pre} utilized pre-training pipelines. \citet{he2024looking} further extended this to retrospective synthesis. These methods mitigate the ``incremental mode coverage'' challenge described in Section~\ref{sec:increment_learning}. By decomposing the search space or conditioning on specific goals, these methods increase the local contrast ratio $\Lambda_{\mathcal{X}}$ in Proposition~\ref{thm:contrast_ratio}. Theoretically, this ensures high-reward states maintain non-negligible probability, preventing loss explosions.
\end{itemize}

\section{More Details of Stable GFlowNets}
\label{app:stable_gfn}

\paragraph{Adaptive threshold selection.}
The loss threshold $c$ controls the extent of stabilization by capping the effective training loss. Since the typical scale of $\mathcal{L}_{\mathrm{TB}}$ is problem-dependent (e.g., reward scale, environment size, and policy entropy), a fixed global choice is brittle. We therefore treat $c$ as an adaptive parameter and update it online using an exponential moving-average rule, analogous to the soft target updates used in reinforcement learning~\citep{lillicrap2015continuous}. Concretely, at iteration $t$, we set
\begin{equation}
c_{\text{new}} \leftarrow (1-\beta)\,c_{\text{old}} + \beta\, \max_{\tau\in\mathcal{T}_{batch}} \sqrt{\mathcal{L}_{\mathrm{TB}}(\tau)},
\label{eq:c_update}
\end{equation}
where $\beta \in (0,1)$ is a small smoothing coefficient (we use $\beta=0.05$ for our experiments). This update tracks the prevailing loss scale while avoiding abrupt changes, ensuring that stabilization remains neither overly conservative (too small $c$) nor inactive (too large $c$). The aggregation operator is chosen heuristically, and Appendix~\ref{sec:app_target_threshold} empirically compares max, mean, and median aggregation, as well as training without reference flow.

\paragraph{Computing $\mathcal{M}_{\text{TV}}$ and $\mathcal{B}_{\text{TV}}$ via 1D optimization.}
We compute the probabilistic certificate $\mathcal{B}_{\text{TV}}$ (and similarly the estimated TV bound $\mathcal{M}_{\text{TV}}$ from Theorem~\ref{thm:resolution_reference_flow}) by exploiting a one-dimensional structure. For a fixed threshold $c$, the bound is monotone in the reference-flow magnitude, so tightening the certificate reduces to a bounded scalar optimization over that single degree of freedom. In practice, we use \texttt{scipy.optimize.minimize\_scalar} with a bounded search interval. The interval's upper bound is $\max_i \sqrt{\mathcal{L}_{\text{TB}}(\tau_i)}$, the lower bound is obtained from $c>0$ and the condition $\frac{\delta_c(\tau_i)}{R(\tau_i)}<\frac{1}{\exp(c)-1}$ for every trajectory $\tau_i$, which leads to: 

$$
c > \max_{i\in\{i \mid ZP_F(\tau_i)>2R(\tau_i)\}} \log\left(\frac{ZP_F(\tau_i)}{R(\tau_i)} - 1\right) 
$$


\paragraph{Full algorithm.} The full algorithm is as below. When global backward sampling is available, replace $\mathcal{X}_{sub}$ with the full terminal state space $\mathcal{X}$ and omit the subset-update step.

\begin{algorithm}
\footnotesize
\caption{Stable GFlowNets with TV Monitoring and Probabilistic Guarantees}
\label{alg:stabilized_gflownet}
\begin{algorithmic}[1]
\Require TV target $d$; confidence $1-2\alpha$; initial threshold $c$;
patience $W$; subset size $K$; batch sizes $m,n$;
maximum rounds $T_{\max}$.
\State Initialize $\mathcal{B}_{\text{TV}}\gets 1$,
$\mathcal{X}_{sub}\gets\emptyset$, $w\gets 0$, and $t\gets 0$.

\While{$\mathcal{B}_{\text{TV}}>d$ and $t<T_{\max}$}
    \State $t\gets t+1$; $b_1\gets\infty$.
    \State Sample $m$ forward trajectories $\mathcal{T}_F$ from $P_F$.
    \State $\mathcal{X}_{old}\gets\mathcal{X}_{sub}$.
    \State Update $\mathcal{X}_{sub}$ by retaining up to $K$ distinct
    highest-reward states from
    $\mathcal{X}_{old}\cup\{x_\tau:\tau\in\mathcal{T}_F\}$,
    where $x_\tau$ is the terminal state of $\tau$.
    \State Sample $n$ backward trajectories $\mathcal{T}_B$:
    draw $x\in\mathcal{X}_{sub}$ with probability
    $R(x)/\sum_{y\in\mathcal{X}_{sub}}R(y)$,
    then draw $\tau\sim P_B(\cdot\mid x)$.
    \State $w\gets
    \mathbf{1}\{\mathcal{X}_{sub}=\mathcal{X}_{old}\}(w+1)$.

    \If{$w\ge W$}
        \State Compute $\mathcal{B}_{\text{TV}}$ at confidence $1-2\alpha$
        using Thm.~\ref{thm:probablistic_TV_bound_w_reference_flow}
        for trajectories ending in $\mathcal{X}_{sub}$.
        \State Set $b_1$ to the first term of Eq.~\eqref{equ:bet_TV}.
        \If{$\mathcal{B}_{\text{TV}}\le d$}
            \State \textbf{break}
        \EndIf
    \EndIf

    \If{$b_1<d$}
        \State Keep the model fixed and accumulate additional
        certification samples.
    \Else
        \State Compute $\mathcal{L}_{TB}$ on
        $\mathcal{T}_F\cup\mathcal{T}_B$.
        \State Update $c$ using Eq.~\eqref{eq:c_update}.
        \State Compute the reference flow using
        Eq.~\eqref{equ:reference_flow}.
        \State Compute $\mathcal{L}_{aug}$ using
        Eq.~\eqref{equ:augmented_loss}.
        \State Update $P_F$, $P_B$, and $\log Z$
        by minimizing $\mathcal{L}_{aug}$.
        \State $\mathcal{B}_{\text{TV}}\gets 1$.
        \Comment{Invalidate the pre-update certificate}
    \EndIf
\EndWhile

\State Compute $\mathcal{M}_{\text{TV}}$ for saved checkpoints over
$\mathcal{X}_{sub}$ using Eq.~\eqref{equ:mc_tv_bound}.
\State \Return $P_F$, $P_B$, $\log Z$,
$\mathcal{B}_{\text{TV}}$, $\mathcal{M}_{\text{TV}}$, $\mathcal{X}_{sub}$.
\end{algorithmic}
\end{algorithm}

\section{Experimental Details}\label{app:environment}
\subsection{Implementations}
We implement GFlowNet training with DB, FM, TB, and SubTB losses following \citet{lahlou2023torchgfn}. For adaptive teacher networks, we adopt the implementation of \citet{kim2024adaptive}. For weighted detailed balance (WDB), we follow \citet{silva2025gflownets} by reweighting each transition’s loss inversely by the number of terminal states reachable from that transition, and then normalizing the weights within the sampled trajectory so they sum to 1. We parameterize $P_F$, $P_B$, and the flow function $F$ with the same MLP architecture but separate parameters. 

For optimization stability, we follow \citet{shen2023towards} and clip gradient norms to $10.0$ and clamp policy logits to $[-50,50]$. Although gradient clipping is not enabled in the original torchgfn~\footnote{\url{https://github.com/GFNOrg/torchgfn/blob/master/tutorials/examples/train_hypergrid.py}} code, we found it consistently improves baseline stability and performance; we therefore apply it to all methods to ensure a fair comparison and to isolate the gains from Stable GFlowNets. 
With gradient clipping, we also find that model performance is less sensitive to partition-function initialization, so we fix initial $\log Z_{\theta}=0$ and do not tune it across environments.

For the L14-RNA1 task, we use a reward-prioritized replay buffer of size $1,000$ following \citet{kim2024adaptive}. We also adopt $\epsilon$-greedy exploration~\citep{jain2022biological} with $\epsilon=0.05$: with probability $0.05$, the forward policy takes a uniformly random action instead of sampling from $P_{F_\theta}$.

For the sEH environment~\citep{bengio2021flow}, we use a molecule-graph fragment MDP where the agent sequentially attaches fragments from a library of 72 unique blocks, growing molecules up to 8 blocks. 
Let $R_{\text{raw}}(x)$ denote the untransformed output of the
pretrained sEH proxy. The reward used for training is
$$
R(x)
=
\left(
\frac{\max\{R_{\mathrm{raw}}(x),0.01\}}{8}
\right)^{10}
$$
For evaluation, we define a mode as a distinct generated
molecule satisfying $R_{\mathrm{raw}}(x)\geq 7.5$, with
uniqueness determined by canonical SMILES. The threshold
$7.5$ applies to the raw proxy output, not to the transformed
training reward $R(x)$. Scaffold counts and
Tanimoto-separated counts are computed from the same set
of qualifying, canonical-unique molecules. To ensure full reproducibility, we enforce deterministic PyTorch and CUDA execution, evaluate the pretrained PyTorch Geometric proxy on the CPU, and verify exact agreement across five repeated runs.

All RQ2 experiments are repeated over five random seeds ${470, 3825, 4444, 8888, 9528}$. For RQ1, we use the results obtained by seed 470. For RQ3, we run the extended L14-RNA1 experiments with seed 0 due to the limit of our computational resources.

\subsection{Hyperparameters}
\label{app:hyperparams}
For Regular Tree and Hypergrid, we use an MLP with two hidden layers of 256 units each; the Hypergrid architecture matches the setting in \citet{malkin2022trajectory}. For L14-RNA1, we follow \citet{kim2024adaptive} and use a hidden size of 128 for short runs (10{,}000 training rounds). Prior work typically evaluates over fewer rounds, but in our longer runs we find wider networks perform better; therefore for extended training (300{,}000 rounds), we use a hidden size of 256. For sEH, we parameterize both forward and backward policies using a 3-layer MLP with 256 hidden units and LeakyReLU activations.

We use a batch size of 32. The learning rate for $\log Z_\theta$ is set to be 100$\times$ the learning rate of the forward and backward policies. Hyperparameters are selected by grid search: policy learning rates in $\{10^{-4}, 5\times 10^{-4}, 10^{-3}\}$ and activation functions in {ReLU, LeakyReLU}. We use a learning rate of $10^{-3}$ for Regular Tree, $10^{-4}$ for Hypergrid and L14-RNA1, and $5\times 10^{-4}$ for sEH, with LeakyReLU in all experiments.

For Stable GFlowNets, we split each batch evenly: half (i.e., 16 trajectories) is from forward sampling and half from backward sampling. We set the initial loss threshold for stabilization to $c=10$, and the moving-average coefficient to $\beta=0.05$. We set the TV target to $d=0.01$ and the confidence level to $0.9$ ($\alpha=0.05$), which is relatively strict. For L14-RNA1, we set the patience parameter to 10 and the high-reward state size to $|\mathcal{X}_{\text{sub}}|=10{,}000$. Under these settings, only the simplest Regular Tree environment terminates early when the certification condition is met.


\paragraph{Sampling composition.} 
In RQ2, Stable uses 16 forward-sampled and 16 backward-sampled
trajectories per round. Its internal ablation, TB + backward
sampling, uses the same sampling mixture but disables
reference-flow stabilization. StableTeacher uses 16 forward,
8 backward, and 8 teacher-sampled trajectories per round.
Thus, backward and teacher samples replace part of the
training batch rather than increasing its size. For the
extended L14-RNA1 experiments, StableTeacher uses 48 trajectories
per round, comprising 16 samples from each source.

\subsection{Computational Resources}
\label{app:hardware}
All experiments were conducted on an internal compute cluster managed by the SLURM workload manager, utilizing reproducible job scripts. Each compute node was equipped with a single NVIDIA A10 GPU (24GB VRAM) and an AMD EPYC 7413 24-core CPU. The software environment relied on Python 3.10.18 and PyTorch 2.5.1, compiled with CUDA 12.4.

\section{Additional Experimental Results}\label{app:additional_results}

\subsection{Additional Results for RQ1: Learned patterns in the Hypergrid environment in Figure~\ref{fig:exp1_result}}
\label{app:hypergrid_explain}

Figures~\ref{fig:app_db_patterns}-\ref{fig:app_tb_patterns} visualize samples from the trained models at different
training stages. These patterns offer a possible explanation
for the larger fluctuations and higher max-to-rest loss
ratios observed in the $H=16$ Hypergrid. More frequent
visits to high-reward states may expose the model to
poorly fitted trajectories more often. Meanwhile, TB
typically concentrates its samples in a single quadrant
of the displayed projections, leaving high-reward regions
elsewhere poorly represented. Trajectories reaching these
underexplored regions can therefore produce large training
losses. Overall, these observations suggest that instability
depends not only on reward sparsity but also on the process
of discovering additional high-reward modes: once an existing
mode is well fitted, encountering a previously underexplored
mode can produce a loss that dominates the batch, causing
the max-to-rest ratio to spike.


\begin{figure*}[ht]
\centering
\includegraphics[width=0.7\textwidth]{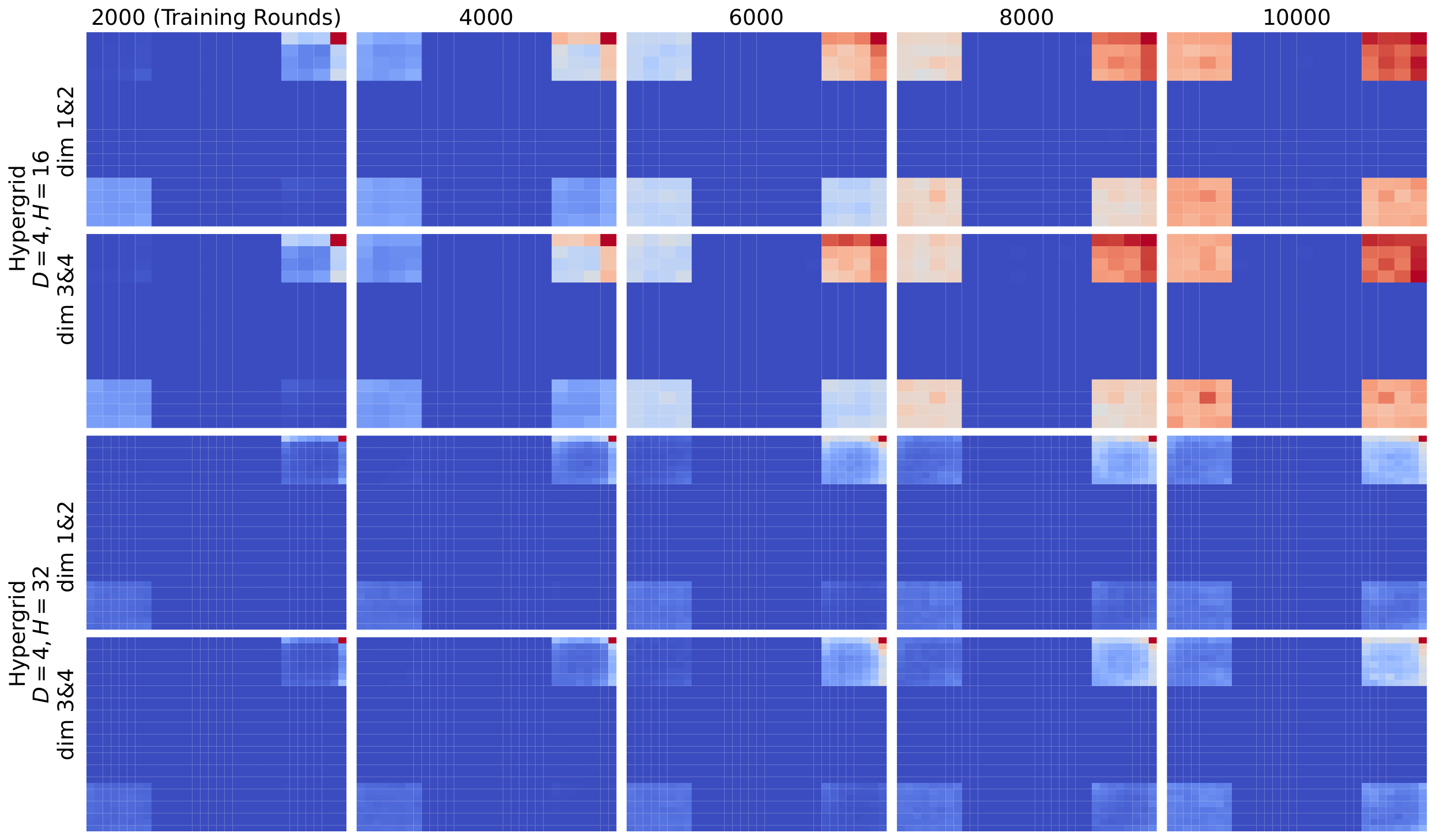}
\caption{Learned patterns of DB. Each subfigure uses 100{,}000 samples.}
\label{fig:app_db_patterns}
\end{figure*}

\begin{figure*}[ht]
\centering
\includegraphics[width=0.7\textwidth]{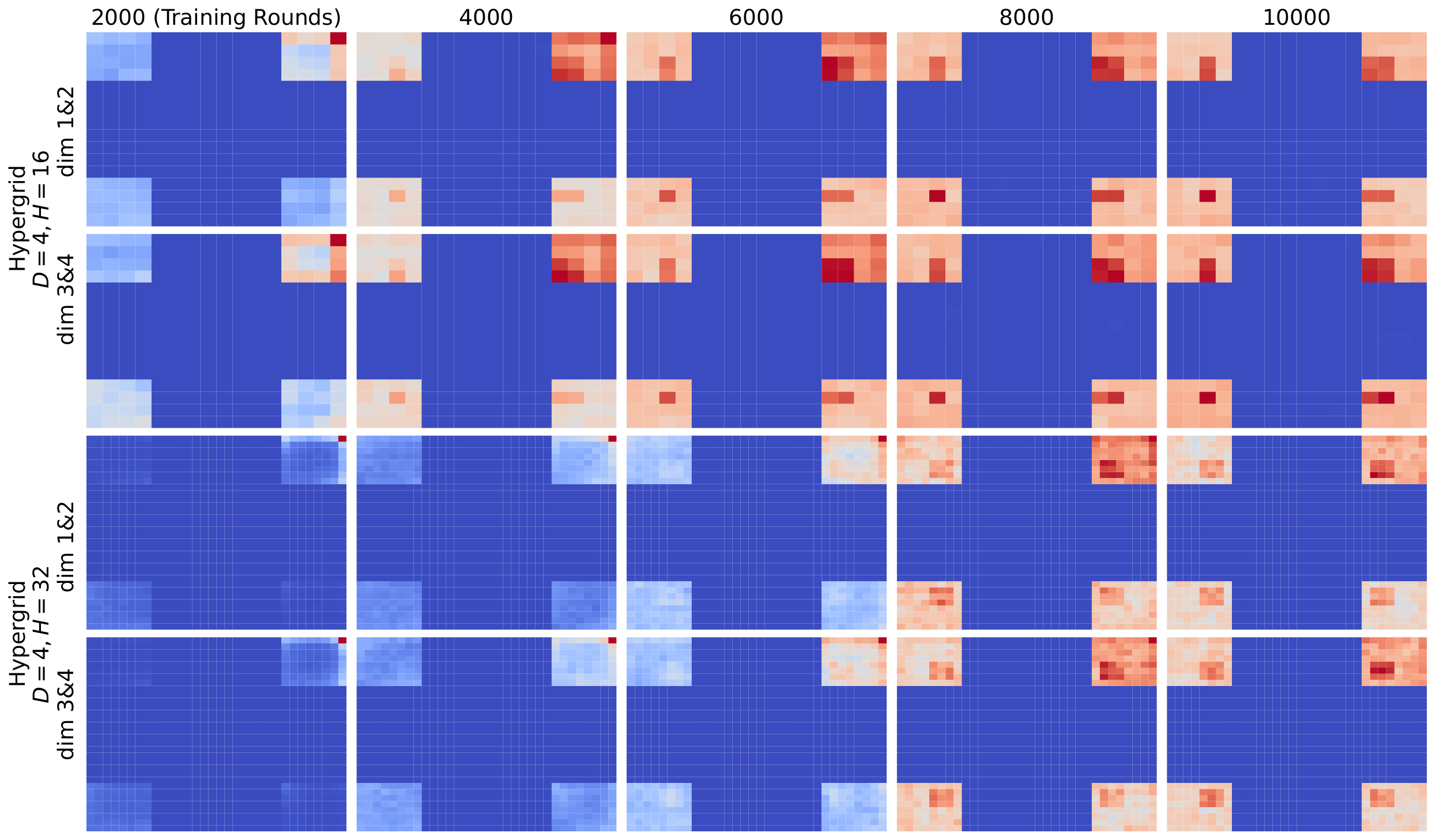}
\caption{Learned patterns of FM. Each subfigure uses 100{,}000 samples.}
\label{fig:app_fm_patterns}
\end{figure*}

\begin{figure*}[ht]
\centering
\includegraphics[width=0.7\textwidth]{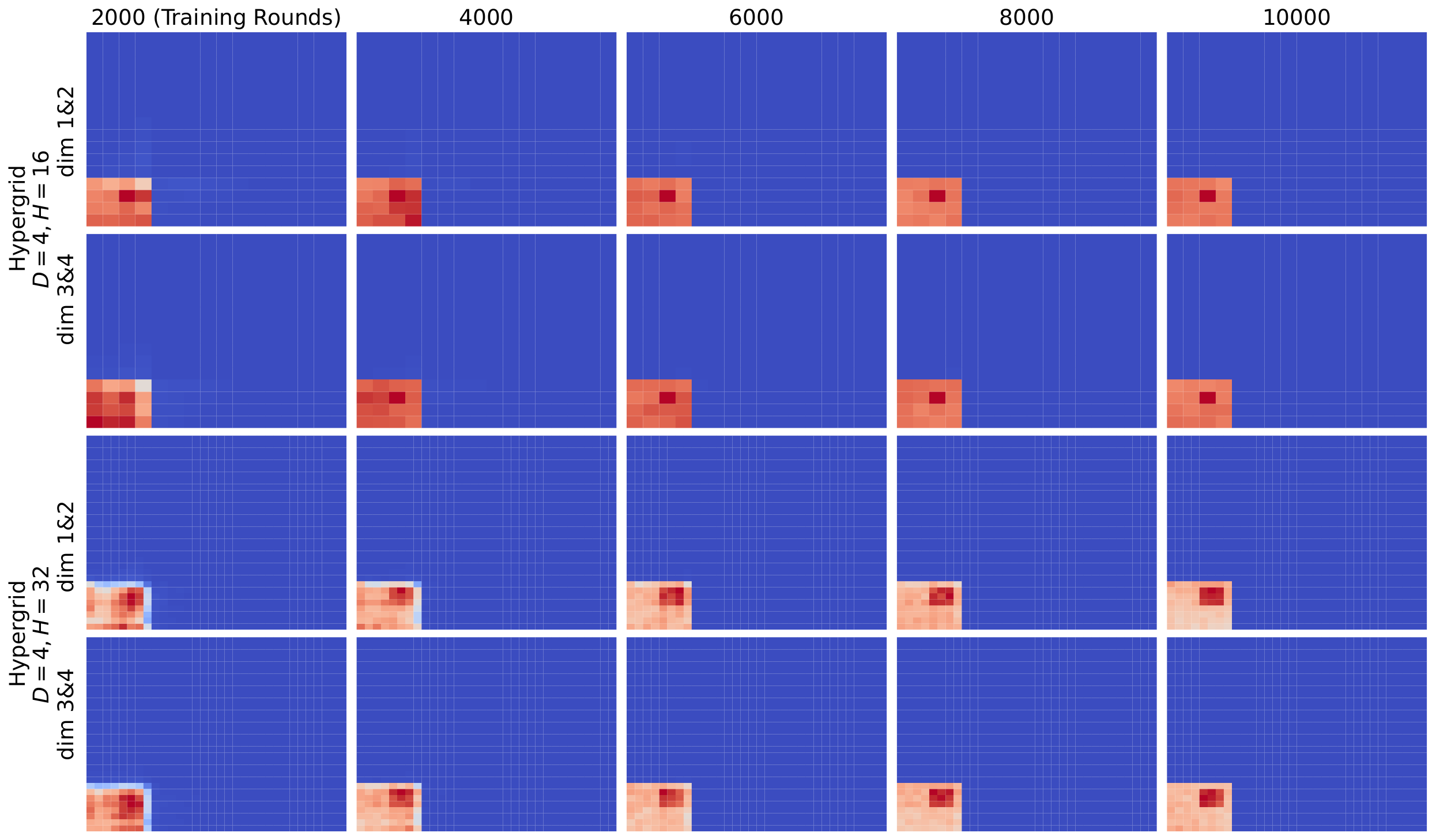}
\caption{Learned patterns of TB. Each subfigure uses 100{,}000 samples.}
\label{fig:app_tb_patterns}
\end{figure*}

\subsection{Additional Results for RQ2}

\subsubsection{Ablating the Reference Flow in Stable GFlowNets}
\label{sec:app_target_threshold}
We ablate the reference-flow loss used in Stable GFlowNets by comparing different aggregation rules for updating the adaptive loss threshold in Equation~(\ref{eq:c_update}). 
Specifically, we consider the maximum, mean, and median batch losses, as well as a \emph{None} variant that removes the reference-flow loss while keeping the remaining training procedure unchanged. 
This ablation is conducted under the RQ2 setting. 
For L14-RNA1, we use a larger MLP with hidden-layer widths 256/256, whereas Table~\ref{tab:performance} uses 128/128.

Overall, the maximum and mean aggregations provide the most reliable behavior. 
On Hypergrid, both achieve strong performance, with the mean slightly better at $H=16$ and the maximum slightly better at the harder $H=32$ setting. 
On sEH, mean aggregation yields slightly higher average mode coverage than maximum aggregation, but with substantially greater variability across seeds. 
These results suggest that effective stabilization requires
balancing the attenuation of extreme losses with the preservation
of useful learning signals. Maximum and mean aggregation remain
responsive to rare, high-loss trajectories, allowing the threshold
to adapt to their magnitude while reference flow limits residuals
that exceed the threshold. In contrast, median aggregation can
keep the threshold low when most trajectories have small residuals,
resulting in stronger attenuation of the remaining high-loss
trajectories. This may suppress useful signals for learning
poorly fitted modes, providing a possible explanation for the
median variant's poor performance on Hypergrid and sEH.

The L14-RNA1 results suggest a stability--exploration
trade-off. Removing reference flow yields the highest
average number of modes discovered during training,
but also produces a substantially larger Peak Max-to-Rest
Ratio, indicating that individual trajectories can dominate
the batch loss. Thus, the \emph{None} variant may favor
exploration by preserving stronger learning signals from
poorly fitted trajectories, at the cost of greater loss
imbalance. This interpretation is consistent with
Figure~\ref{fig:exp1_result} and Table~\ref{tab:performance}.
On L14-RNA1, where loss imbalance is relatively mild
compared with sEH, stronger attenuation may make learning
more conservative and limit mode discovery. In contrast,
reference-flow stabilization provides substantial benefits
on the much larger sEH task, where extreme loss imbalance
is more pronounced. Together, these observations suggest
that the value of stabilization depends on balancing
useful exploratory signals against the disruptive effects
of highly concentrated losses.

Taken together, these results support the use of maximum or mean aggregation for the adaptive threshold. 
Using the maximum yields a slightly more conservative reference-flow adjustment, performs best on the hardest Hypergrid setting, and achieves lower variance on sEH, whereas the mean provides comparable performance with smoother threshold updates.
The \emph{None} variant can improve exploration on L14-RNA1, but it sacrifices the stabilizing mechanism needed for robust behavior across environments.

\begin{table}[t]
\setlength{\tabcolsep}{3pt}
\caption{Effect of adaptive loss-threshold aggregation on Hypergrid, L14-RNA1, and sEH. 
We compare maximum, mean, and median aggregation rules for updating the adaptive threshold in Equation~(\ref{eq:c_update}), together with a \emph{None} variant that removes the reference-flow loss. 
We report the Hypergrid empirical total $L_1$ distance, L14-RNA1 mode counts discovered during training, the Peak Max-to-Rest Ratio, and sEH training scaffold counts. 
The Max-to-Rest Ratio is defined as the largest trajectory loss in a batch divided by the sum of the remaining trajectory losses; we report its peak value over training. 
Hypergrid evaluations use $10^5$ samples from the final forward policy $P_F$ (mean $\pm$ std over 5 seeds). 
For L14-RNA1, we use a larger MLP with hidden-layer widths 256/256, whereas Table~\ref{tab:performance} uses 128/128.}
\label{tab:performance2}
\begin{center}
\resizebox{\linewidth}{!}{
\begingroup
\setlength{\tabcolsep}{5pt}
\renewcommand{\arraystretch}{1.12}
\resizebox{\linewidth}{!}{%
\begin{tabular}{@{}l*{5}{r@{\,$\pm$\,}l}@{}}
\toprule
& \multicolumn{4}{c}{%
    \makecell{\textbf{Hypergrid}\\
    (Empirical Total $L_1$ $\downarrow$)}}
& \multicolumn{4}{c}{\textbf{L14-RNA1}}
& \multicolumn{2}{c}{\textbf{sEH}} \\
\cmidrule(lr){2-5}
\cmidrule(lr){6-9}
\cmidrule(lr){10-11}
\textbf{Method}
& \multicolumn{2}{c}{$D=4, H=16$}
& \multicolumn{2}{c}{$D=4, H=32$}
& \multicolumn{2}{c}{%
    \makecell{$\#$ modes $\uparrow$}}
& \multicolumn{2}{c}{%
    \makecell{Peak Max-to-\\Rest Ratio}}
& \multicolumn{2}{c}{%
    \makecell{$\#$ scaffolds $\uparrow$}} \\
\midrule
Max
& 0.290 & {\small 0.002}
& \textbf{0.713} & \textbf{{\small 0.002}}
& 1754.2 & {\small 123.0}
& 11.5 & {\small 2.0}
& 14142.6 & {\small 2388.2} \\
Mean
& \textbf{0.287} & \textbf{{\small 0.003}}
& 0.730 & {\small 0.001}
& 1745.4 & {\small 37.6}
& 9.5 & {\small 2.4}
& \textbf{14721.4} & \textbf{{\small 4974.3}} \\
Median
& 1.885 & {\small 0.000}
& 1.875 & {\small 0.000}
& 1744.6 & {\small 43.6}
& 16.8 & {\small 13.2}
& 1283.6 & {\small 1010.0} \\
\makecell[l]{None {\small (TB + backward sampling)}}
& 0.316 & {\small 0.003}
& \textbf{0.713} & \textbf{{\small 0.003}}
& \textbf{1830.4} & \textbf{{\small 58.7}}
& 81.3 & {\small 39.3}
& 8680.2 & {\small 2254.2} \\
\bottomrule
\end{tabular}%
}
\endgroup
}
\end{center}
\vskip -0.1in
\end{table}

\subsubsection{Integrate Adaptive Teacher and Stable GFlowNets}
\label{sec:app_teacher_stable}
We evaluate a simple integration of Stable GFlowNets with the Adaptive Teacher method of \citet{kim2024adaptive}. In addition to forward- and backward-sampled trajectories, we include a third set of teacher-sampled trajectories. Each batch contains 16 forward, 16 backward, and 16 teacher trajectories (batch size 48). We use the same batch size for all baselines to ensure a fair comparison.

Figure~\ref{fig:rna14_results} reports the long-run results on L14-RNA1. StableTeacher discovers 8,858 modes during training, covering 98.78\% of the modes in this environment while the total
number of sampled trajectories, including repeated samples,
amounts to only 5.36\% of the number of possible terminal states. Measured by the total reward mass accumulated by the Top-10,000 terminal states, StableTeacher also achieves higher values than Teacher early in training. The smoothed loss curves indicate that StableTeacher responds systematically when new modes are uncovered, while maintaining a more balanced max-to-rest ratio throughout training. In contrast, Teacher shows an increasing max-to-rest ratio later in training, suggesting that a small number of trajectories dominate the batch as high-reward modes emerge, which is consistent with less stable optimization.

We also compare TB against Stable GFlowNets in long runs. With gradient clipping, TB performs well early on but tends to saturate, becoming less likely to discover additional modes. Its training loss continues to decrease, yet exploration slows. Stable GFlowNets, in contrast, continue to adapt as new modes appear: their loss remains higher than TB in later stages, reflecting active correction driven by newly discovered high-reward outcomes, whereas standard TB tends to settle into its ``comfort zone''. We observe a concurrent increase in the max-to-rest ratio, which we hypothesize is transient as the loss threshold increases and will subside once the newly discovered modes are sufficiently represented in the training distribution.

Collectively, these results suggest that Stable GFlowNets enables stronger incremental mode coverage, complementing aggressive exploration in settings where sampling from the full set of terminal states is difficult.

\begin{figure*}[!ht]
\centering
\includegraphics[width=\textwidth]{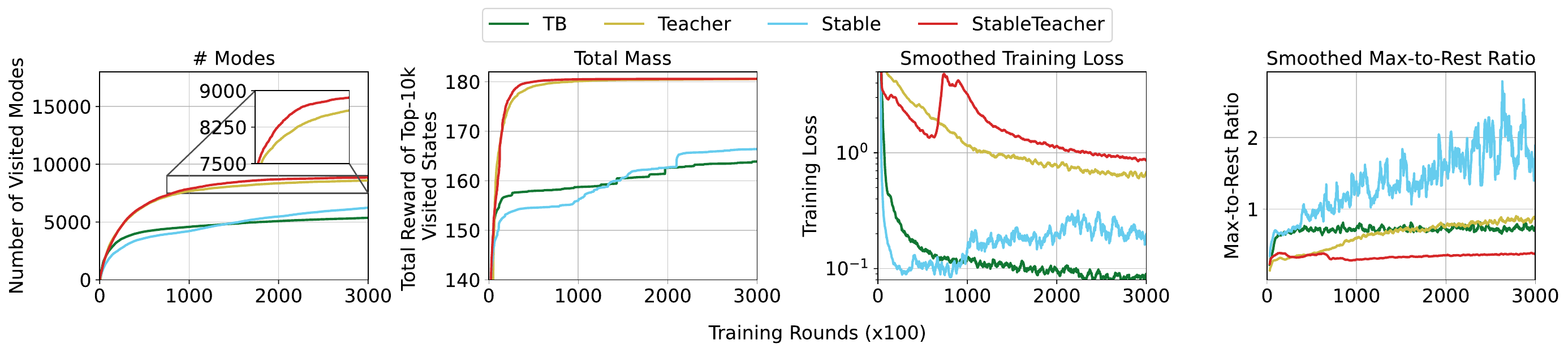}
\caption{Extended experiment results for L14-RNA1. Total mass represents the cumulative reward captured by the top-$K$ high-reward states discovered during training.}
\label{fig:rna14_results}
\end{figure*}

\subsubsection{sEH Mode Counts under Different Definitions}

Table~\ref{tab:molecular_diversity} reports three complementary molecular diversity metrics for sEH: the number of unique modes, the number of distinct Bemis-Murcko scaffolds, and the number of Tanimoto-separated molecules.

For the \textbf{\# modes} metric, we define a mode as a uniquely generated molecule with raw proxy reward $R_{\text{raw}}(x)\geq7.5$, where uniqueness is determined by canonical SMILES rather than Tanimoto similarity. Specifically, we retain nonempty terminal-molecule SMILES satisfying $R_{\text{raw}}(x)\geq7.5$, parse them using \texttt{Chem.MolFromSmiles}, and canonicalize and deduplicate them using \texttt{Chem.MolToSmiles(mol, canonical=True)}.

For the \textbf{Bemis-Murcko scaffold} metric, we compute the scaffold of each canonical-unique molecule using \texttt{MurckoScaffold.MurckoScaffoldSmiles(mol=mol, includeChirality=False)} and report the number of distinct scaffold SMILES. Under RDKit's standard representation, two molecules share a scaffold when removing their side chains yields the same typed ring-and-linker structure, irrespective of chirality.

The \textbf{Tanimoto <0.7} metric is computed independently. Canonical-unique molecules are processed in encounter order using \texttt{Chem.RDKFingerprint}, and a molecule is retained only if its Tanimoto similarity is strictly below 0.7 with respect to every previously retained representative.

As shown in Table~\ref{tab:molecular_diversity}, Stable substantially outperforms all baselines under each definition of molecular diversity. It discovers $14142.6\pm 2388.2$ unique modes and Bemis-Murcko scaffolds, together with $3692.4\pm 938.7$ Tanimoto-separated molecules. This consistent advantage across different diversity criteria indicates that the improvement is not an artifact of a particular diversity definition. The per-seed discovery traces (Figure~\ref{fig:cross_seed_trace}) further show that mode discovery accelerates as training progresses. 

\begin{table}[t]
\setlength{\tabcolsep}{3pt}
\centering
\caption{Molecular mode and diversity metrics for sEH. Results are reported as mean $\pm$ standard deviation.}
\label{tab:molecular_diversity}
\resizebox{0.8\linewidth}{!}{
\begingroup
\setlength{\tabcolsep}{6pt}
\renewcommand{\arraystretch}{1.12}
\resizebox{\linewidth}{!}{%
\begin{tabular}{@{}l*{3}{r@{\,$\pm$\,}l}@{}}
\toprule
\textbf{Method}
& \multicolumn{2}{c}{\textbf{\# modes}}
& \multicolumn{2}{c}{%
    \makecell{\textbf{Bemis--Murcko}\\\textbf{scaffolds}}}
& \multicolumn{2}{c}{\textbf{Tanimoto $< 0.7$}} \\
\midrule
TB
& 137.2 & {\small 77.9}
& 137.2 & {\small 77.9}
& 55.2 & {\small 22.8} \\
DB
& 32.2 & {\small 12.5}
& 32.2 & {\small 12.5}
& 16.8 & {\small 5.6} \\
SubTB
& 68.4 & {\small 41.9}
& 68.4 & {\small 41.9}
& 35.8 & {\small 14.5} \\
WDB
& 3.2 & {\small 2.8}
& 3.2 & {\small 2.8}
& 2.8 & {\small 1.9} \\
Teacher
& 5.6 & {\small 5.9}
& 5.6 & {\small 5.9}
& 5.2 & {\small 5.1} \\
\midrule
\textbf{Stable}
& \textbf{14142.6} & \textbf{{\small 2388.2}}
& \textbf{14142.6} & \textbf{{\small 2388.2}}
& \textbf{3692.4} & \textbf{{\small 938.7}} \\
\textbf{StableTeacher}
& 3728.8 & {\small 4197.8}
& 3722.2 & {\small 4187.8}
& 457.8 & {\small 403.7} \\
\textbf{TB + backward sampling}
& 8680.6 & {\small 2254.2}
& 8680.2 & {\small 2254.2}
& 2058.4 & {\small 402.7} \\
\bottomrule
\end{tabular}%
}
\endgroup}
\end{table}

\begin{figure}[!ht]
    \centering
    \includegraphics[width=0.7\linewidth]{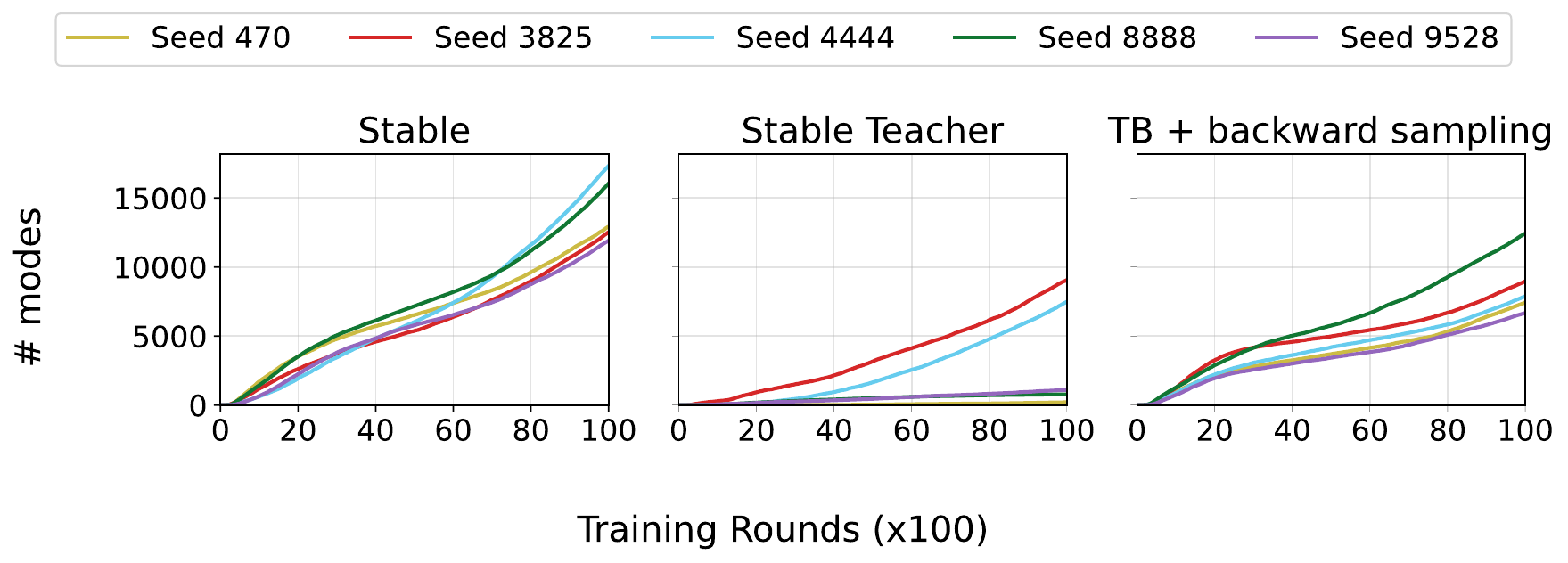}
    \caption{Per-seed molecular mode discovery for sEH. The per-seed traces of Stable GFlowNets show progressively faster mode discovery as training progresses.}
    \label{fig:cross_seed_trace}
    \vskip -0.1in
\end{figure}

\subsection{Additional Results for RQ3}

\label{app:theorem_additional_results}

We compare FCS and $\mathcal{M}_{\mathrm{TV}}$ as monitoring
signals for distributional error, rather than as competing
TV certificates. FCS measures flow consistency and corresponds
to one component of an existing TV bound, whereas
$\mathcal{M}_{\mathrm{TV}}$ is a Monte Carlo estimate of the
bound. Neither quantity, as evaluated
here, provides a finite-sample confidence guarantee.
We therefore assess how well they track changes in true TV,
including their correlation across training checkpoints,
their sensitivity to subgraph size and sampling budget,
and their predictive error where calibration is applied.
Formal probabilistic certification is instead provided by
$\mathcal{B}_{\mathrm{TV}}$.

Table~\ref{tab:checkpoint_level_values} reports the values at all ten training checkpoints. In most environments, $\mathcal{M}_{\text{TV}}$ follows the overall trend of the true TV error, and the calibrated $\mathcal{M}_{\text{TV}}$ is closer in scale. For example, in HyperGrid ($D=4, H=32$), it captures both the general decrease in true TV and the increase at the 80\% checkpoint. In L14-RNA1, it also reflects the increases from 40\% to 50\% and from 60\% to 70\%.

\begin{table}[t]
\setlength{\tabcolsep}{3pt}
\centering
\caption{Checkpoint-level values underlying Table~\ref{tab:correlation_calibration}. Columns correspond to ten evenly spaced checkpoints. Calibrated $\mathcal{M}_{\mathrm{TV}}$ uses an environment-specific affine regression fitted on the first five checkpoints and then frozen: its first five entries are in-sample fitted values, whereas its last five entries are held-out predictions.}
\label{tab:checkpoint_level_values}
\resizebox{\linewidth}{!}{%
\begin{tabular}{llrrrrrrrrrr}
\toprule
\multirow{2}{*}{Environment}
& \multirow{2}{*}{Metric}
& \multicolumn{10}{c}{Checkpoint} \\
\cmidrule(lr){3-12}
& & 10\% & 20\% & 30\% & 40\% & 50\%
  & 60\% & 70\% & 80\% & 90\% & 100\% \\
\midrule

\multirow{5}{*}{RegularTree $D=5$}
& $\mathcal{B}_{\mathrm{TV}}$
& 0.1554 & 0.1347 & 0.1260 & 0.1169 & 0.1086
& 0.0867 & 0.0747 & 0.0747 & 0.0747 & 0.0747 \\
& $\mathcal{M}_{\mathrm{TV}}$
& 0.0030 & 0.0017 & 0.0010 & 0.0006 & 0.0003
& 0.0002 & 0.0003 & 0.0003 & 0.0003 & 0.0003 \\
& Calibrated $\mathcal{M}_{\mathrm{TV}}$
& 0.0010 & 0.0006 & 0.0003 & 0.0002 & 0.0001
& $8.77{\times}10^{-5}$ & $9.50{\times}10^{-5}$
& $9.50{\times}10^{-5}$ & $9.50{\times}10^{-5}$
& $9.50{\times}10^{-5}$ \\
& FCS
& 0.0012 & 0.0007 & 0.0004 & 0.0003 & 0.0002
& 0.0001 & $6.32{\times}10^{-5}$ & $6.32{\times}10^{-5}$
& $6.32{\times}10^{-5}$ & $6.32{\times}10^{-5}$ \\
& True TV
& 0.0010 & 0.0006 & 0.0003 & 0.0002 & 0.0001
& $7.85{\times}10^{-5}$ & $6.35{\times}10^{-5}$
& $6.35{\times}10^{-5}$ & $6.35{\times}10^{-5}$
& $6.35{\times}10^{-5}$ \\
\midrule

\multirow{5}{*}{RegularTree $D=10$}
& $\mathcal{B}_{\mathrm{TV}}$
& 0.3009 & 0.1042 & 0.0914 & 0.0864 & 0.0839
& 0.0826 & 0.0805 & 0.0804 & 0.0815 & 0.0802 \\
& $\mathcal{M}_{\mathrm{TV}}$
& 0.1269 & 0.0068 & 0.0035 & 0.0026 & 0.0021
& 0.0020 & 0.0018 & 0.0018 & 0.0016 & 0.0017 \\
& Calibrated $\mathcal{M}_{\mathrm{TV}}$
& 0.0044 & 0.0010 & 0.0010 & 0.0009 & 0.0009
& 0.0009 & 0.0009 & 0.0009 & 0.0009 & 0.0009 \\
& FCS
& 0.0046 & 0.0020 & 0.0010 & 0.0007 & 0.0006
& 0.0005 & 0.0005 & 0.0005 & 0.0005 & 0.0005 \\
& True TV
& 0.0044 & 0.0017 & 0.0009 & 0.0007 & 0.0006
& 0.0005 & 0.0005 & 0.0005 & 0.0004 & 0.0005 \\
\midrule

\multirow{5}{*}{RegularTree $D=15$}
& $\mathcal{B}_{\mathrm{TV}}$
& 1.0000 & 0.1170 & 0.1024 & 0.0890 & 0.0835
& 0.0822 & 0.0824 & 0.0821 & 0.0818 & 0.0812 \\
& $\mathcal{M}_{\mathrm{TV}}$
& 0.7088 & 0.0097 & 0.0056 & 0.0031 & 0.0025
& 0.0020 & 0.0019 & 0.0019 & 0.0017 & 0.0016 \\
& Calibrated $\mathcal{M}_{\mathrm{TV}}$
& 0.0058 & 0.0014 & 0.0014 & 0.0013 & 0.0013
& 0.0013 & 0.0013 & 0.0013 & 0.0013 & 0.0013 \\
& FCS
& 0.0064 & 0.0029 & 0.0015 & 0.0008 & 0.0006
& 0.0005 & 0.0004 & 0.0004 & 0.0004 & 0.0004 \\
& True TV
& 0.0057 & 0.0027 & 0.0014 & 0.0008 & 0.0006
& 0.0005 & 0.0005 & 0.0005 & 0.0004 & 0.0004 \\
\midrule

\multirow{5}{*}{HyperGrid $D=4, H=8$}
& $\mathcal{B}_{\mathrm{TV}}$
& 1.0000 & 1.0000 & 0.7077 & 0.3794 & 0.5045
& 0.4374 & 0.6033 & 0.3363 & 0.3599 & 0.2542 \\
& $\mathcal{M}_{\mathrm{TV}}$
& 0.3994 & 0.1807 & 0.1273 & 0.0776 & 0.0744
& 0.0631 & 0.0640 & 0.0476 & 0.0433 & 0.0498 \\
& Calibrated $\mathcal{M}_{\mathrm{TV}}$
& 0.0933 & 0.0375 & 0.0238 & 0.0111 & 0.0103
& 0.0074 & 0.0077 & 0.0035 & 0.0024 & 0.0040 \\
& FCS
& 0.1197 & 0.0323 & 0.0208 & 0.0135 & 0.0101
& 0.0074 & 0.0085 & 0.0065 & 0.0070 & 0.0058 \\
& True TV
& 0.0945 & 0.0350 & 0.0215 & 0.0142 & 0.0110
& 0.0094 & 0.0100 & 0.0069 & 0.0077 & 0.0064 \\
\midrule

\multirow{5}{*}{HyperGrid $D=4, H=16$}
& $\mathcal{B}_{\mathrm{TV}}$
& 1.0000 & 1.0000 & 1.0000 & 1.0000 & 1.0000
& 1.0000 & 1.0000 & 1.0000 & 1.0000 & 1.0000 \\
& $\mathcal{M}_{\mathrm{TV}}$
& 0.5431 & 0.4037 & 0.3524 & 0.3941 & 0.3531
& 0.2851 & 0.2283 & 0.2243 & 0.1606 & 0.1957 \\
& Calibrated $\mathcal{M}_{\mathrm{TV}}$
& 0.1223 & 0.0966 & 0.0871 & 0.0948 & 0.0873
& 0.0747 & 0.0642 & 0.0635 & 0.0517 & 0.0582 \\
& FCS
& 0.0864 & 0.0530 & 0.0651 & 0.0621 & 0.0377
& 0.0399 & 0.0274 & 0.0278 & 0.0298 & 0.0319 \\
& True TV
& 0.1223 & 0.0989 & 0.0930 & 0.0920 & 0.0819
& 0.0883 & 0.0710 & 0.0662 & 0.0608 & 0.0666 \\
\midrule

\multirow{5}{*}{HyperGrid $D=4, H=32$}
& $\mathcal{B}_{\mathrm{TV}}$
& 1.0000 & 1.0000 & 1.0000 & 1.0000 & 1.0000
& 1.0000 & 1.0000 & 1.0000 & 1.0000 & 1.0000 \\
& $\mathcal{M}_{\mathrm{TV}}$
& 0.8421 & 0.6878 & 0.5429 & 0.4314 & 0.4222
& 0.3791 & 0.3826 & 0.4094 & 0.3189 & 0.3332 \\
& Calibrated $\mathcal{M}_{\mathrm{TV}}$
& 0.1916 & 0.1435 & 0.0984 & 0.0637 & 0.0609
& 0.0475 & 0.0485 & 0.0569 & 0.0287 & 0.0332 \\
& FCS
& 0.2162 & 0.1139 & 0.1203 & 0.0965 & 0.0573
& 0.0490 & 0.0444 & 0.0542 & 0.0453 & 0.0451 \\
& True TV
& 0.2076 & 0.1199 & 0.0940 & 0.0756 & 0.0611
& 0.0515 & 0.0474 & 0.0595 & 0.0381 & 0.0439 \\
\midrule

\multirow{5}{*}{L14-RNA1}
& $\mathcal{B}_{\mathrm{TV}}$
& 1.0000 & 1.0000 & 1.0000 & 1.0000 & 1.0000
& 1.0000 & 1.0000 & 1.0000 & 1.0000 & 1.0000 \\
& $\mathcal{M}_{\mathrm{TV}}$
& 0.4990 & 0.4100 & 0.3465 & 0.2740 & 0.3545
& 0.2631 & 0.3104 & 0.2553 & 0.2273 & 0.2001 \\
& Calibrated $\mathcal{M}_{\mathrm{TV}}$
& 0.1066 & 0.0934 & 0.0840 & 0.0733 & 0.0852
& 0.0716 & 0.0787 & 0.0705 & 0.0663 & 0.0623 \\
& FCS
& 0.0375 & 0.0199 & 0.0302 & 0.0072 & 0.0207
& 0.0132 & 0.0265 & 0.0225 & 0.0143 & 0.0148 \\
& True TV
& 0.1082 & 0.0910 & 0.0942 & 0.0735 & 0.0756
& 0.0537 & 0.0626 & 0.0480 & 0.0381 & 0.0343 \\
\bottomrule
\end{tabular}%
}
\end{table}

Table~\ref{tab:metrics2} further compares the monitoring
behavior of $\mathcal{M}_{\mathrm{TV}}$ and FCS across
subgraph sizes and sampling budgets. To match the information available to $\mathcal{M}_{\text{TV}}$, here, FCS uses Monte Carlo estimates of terminal state flows ($\mathrm{budget}=1$) rather than their exact values. 
Overall, $\mathcal{M}_{\text{TV}}$ is not highly sensitive to the number of sampled trajectories: across settings, it consistently follows the trend of the subgraph True TV. 
Although FCS can be numerically closer to the subgraph True TV in absolute value for small $K$, it increasingly underestimates the model TV as $K$ grows. 
This is because FCS corresponds to only one component of the TV bound in Corollary~2 of~\cite{silva2025gflownets}. 
As $K$ increases, covering the larger subspace with limited samples becomes more difficult, so the induced subspace size $\beta$ remains small relative to $K$ (Table~\ref{tab:subspace_size}). 
Consequently, FCS becomes relatively insensitive to the true size of the certified subgraph, and the gap between FCS and the full TV bound can grow through the second term. 
In contrast, $\mathcal{M}_{\text{TV}}$ tracks the subgraph True TV much more consistently across $K$, achieving a linear-fit $R^2$ above $0.99$ when $m \geq 100$, whereas FCS yields substantially lower $R^2$ values.

\begin{table}[ht]
\setlength{\tabcolsep}{3pt}
\centering
\caption{Robustness to subgraph size for Monte Carlo estimator $\mathcal{M}_{\text{TV}}$, FCS baseline, and subgraph True TV on L14-RNA1 using StableTeacher trained for 300,000 rounds, across varying numbers of sampled trajectories. Percentages indicate the reward mass captured by each top-$K$ subgraph. We also include the ground-truth top-200,000 high-reward states (“Oracle”), which capture 99.93\% of the total rewards. The linear-fit $R^2$ is computed across subgraph sizes
and measures association with subgraph TV.}
\label{tab:metrics2}
\resizebox{0.9\textwidth}{!}{
\begin{tabular}{l c c c c c c}
\toprule
Metric / Setting 
& \makecell{$K=10$\\(28.56\%)} 
& \makecell{$K=10^2$\\(62.57\%)} 
& \makecell{$K=10^3$\\(87.27\%)} 
& \makecell{$K=10^4$\\(97.63\%)} 
& \makecell{Oracle\\(99.93\%)} 
& \makecell{Linear-fit $R^2$ vs. \\Subgraph True TV} \\
\midrule
$\mathcal{M}_{\text{TV}}$ ($m=10$) & $0.2170$ & $0.2783$ & $0.3414$ & $0.4059$ & $0.3915$ & $0.9689$ \\
$\mathcal{M}_{\text{TV}}$ ($m=100$) & $0.2195$ & $0.2897$ & $0.3610$ & $0.4328$ & $0.4544$ & $0.9991$ \\
$\mathcal{M}_{\text{TV}}$ ($m=1000$) & $0.2169$ & $0.2882$ & $0.3711$ & $0.4573$ & $0.4696$ & $0.9952$ \\
FCS (epochs $=1$, bucket $=100$) & $0.0611$ & $0.0875$ & $0.0879$ & $0.0993$ & $0.1001$ & $0.8548$ \\
FCS (epochs $=5$, bucket $=100$) & $0.0704$ & $0.0930$ & $0.0931$ & $0.0986$ & $0.0955$ & $0.6896$ \\
FCS (epochs $=1$, bucket $=1000$) & $0.0613$ & $0.0747$ & $0.0809$ & $0.0803$ & $0.0806$ & $0.7713$ \\
FCS (epochs $=5$, bucket $=1000$) & $0.0577$ & $0.0749$ & $0.0813$ & $0.0809$ & $0.0810$ & $0.7502$ \\
Subgraph True TV & $0.0373$ & $0.0608$ & $0.0857$ & $0.1082$ & $0.1182$ & $1.0000$ \\
\bottomrule
\end{tabular}
}
\end{table}

\begin{table}[h]
\setlength{\tabcolsep}{3pt}
\centering
\caption{Induced subspace size $\beta$ used in FCS scaling as the top-$K$ subgraph grows.}
\label{tab:subspace_size}
\resizebox{\textwidth}{!}{\begin{tabular}{l c c c c c}
\toprule
Metric / Setting & $K=10$ & $K=10^2$ & $K=10^3$ & $K=10^4$ & Oracle \\
\midrule
Subspace size $\beta$ (epochs $=1$, bucket $=1000$) & $10 \pm 0$ & $94.8 \pm 1.3$ & $294.8 \pm 6.7$ & $368.8 \pm 7.2$ & $388.6 \pm 9.9$ \\
\bottomrule
\end{tabular}
}
\end{table}

We also compare $\mathcal{M}_{\mathrm{TV}}$ across
selected training methods. For Hypergrid, we evaluate
policies trained with seed 470 for 30,000 rounds,
using checkpoints every 3,000 rounds. For L14-RNA1,
we use policies trained with seed 0 for 300,000 rounds
and evaluate them every 10,000 rounds.
At each checkpoint, estimation is repeated with
five Monte Carlo seeds ($\{42, 43, 44, 45, 46\}$). Lines and shaded regions
show the mean and sample standard deviation of
these estimates for each fixed policy, rather
than variability across independent training runs.

First, $\mathcal{M}_{\text{TV}}$
can monitor different GFlowNet training methods, rather
than being restricted to Stable GFlowNets. In the simplest Hypergrid
setting ($H=8$), the estimates decrease for all evaluated
methods. On larger Hypergrids, Stable and Teacher achieve
substantial reductions, whereas standard TB shows limited
improvement, with Teacher's gains emerging later at $H=32$.
These observations are consistent with the benefits of directing training on high-reward or poorly fitted trajectories: Stable uses reward-weighted backward sampling, while Teacher targets regions where the student incurs high loss; on L14-RNA1, the estimates also decrease overall, including for TB, under a training setup that uses reward-prioritized replay. Beyond explicitly prioritizing high-reward trajectories (i.e., exploration), SubTB achieves particularly low $\mathcal{M}_{\text{TV}}$, suggesting that effective variance reduction (i.e., stabilization)~\citep{madan2023learning} provides another route to enhance GFlowNets training.


\begin{figure}[!ht]
\centering
\includegraphics[width=\textwidth]{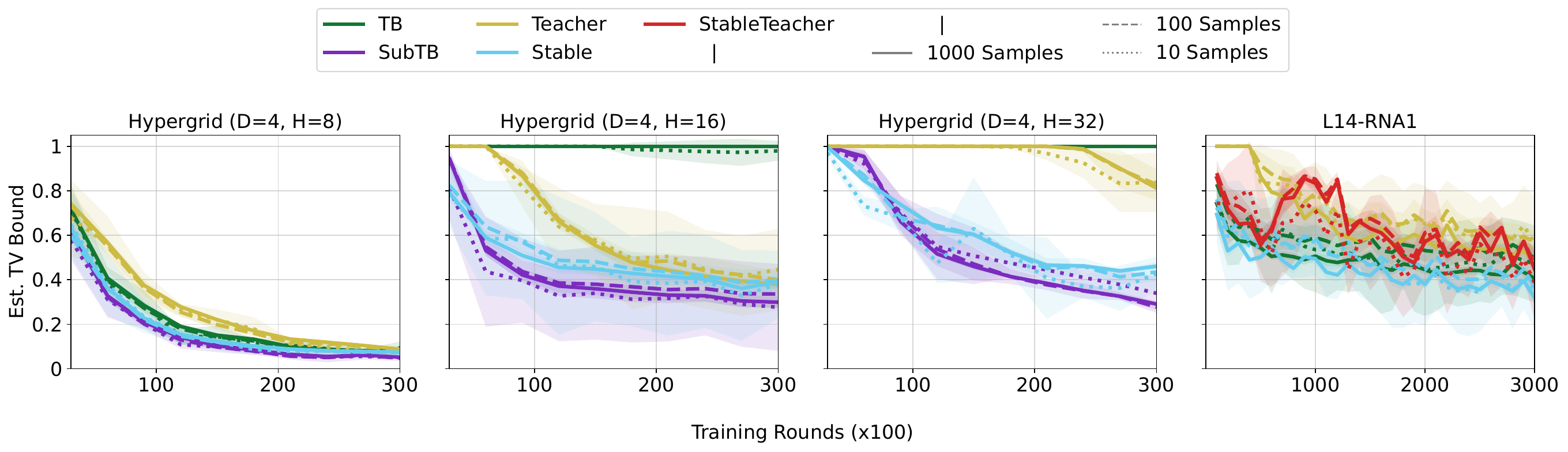}
\caption{Evolution of empirical $M_{\mathrm{TV}}$ estimates (Theorem 3.10). Hypergrid policies (training seed 470) are
evaluated every 3,000 rounds through 30,000; L14-RNA1 policies (training seed 0) every 10,000 through 300,000.
 Lines and shading show the mean $\pm$ one sample standard deviation over five Monte
Carlo seeds for each fixed policy. For RNA, all four methods (TB, Teacher, Stable, and StableTeacher) share a fixed reference consisting
of the deduplicated union of the available terminal state buffer.}
\label{fig:tv_est_results}
\end{figure}

\subsection{Computational Overhead of Adaptive Reference Flow}

We measured the cost of computing the adaptive reference flow relative to the total training-update cost of Stable GFlowNets. Because a Stable GFlowNet training update consists of the standard TB update plus reference-flow computation, the latter directly measures the additional optimization cost over TB. As shown in Table~\ref{tab:compute_time}, the overhead remains modest across all evaluated environments.

\begin{table}[!ht]
\setlength{\tabcolsep}{3pt}
\centering
\caption{Reference-flow computation cost relative to a training update.}
\label{tab:compute_time}
\resizebox{0.9\textwidth}{!}{
\begin{tabular}{lcccc}
\toprule
\shortstack[l]{Reference-flow computation /\\ training update (\%)}
& \shortstack{Hypergrid\\ $D=4$, $H=16$}
& \shortstack{Hypergrid\\ $D=4$, $H=32$}
& L14-RNA1
& sEH \\
\midrule
Wall-time ratio
& $3.21 \pm 0.85$
& $2.83 \pm 0.06$
& $3.50 \pm 0.12$
& $1.31 \pm 0.09$ \\
CPU+GPU-time ratio
& $4.56 \pm 2.06$
& $3.88 \pm 1.49$
& $2.75 \pm 0.14$
& $3.81 \pm 1.88$ \\
\bottomrule
\end{tabular}}
\end{table}